\documentclass{article}

\PassOptionsToPackage{numbers}{natbib}

\usepackage[final]{neurips_2024}

\usepackage[utf8]{inputenc} 
\usepackage[T1]{fontenc}    
\usepackage{hyperref}       
\usepackage{url}            
\usepackage{booktabs}       
\usepackage{amsfonts}       
\usepackage{nicefrac}       
\usepackage{microtype}      
\usepackage{xcolor}         
\usepackage{amsmath}
\usepackage{amsthm}
\usepackage{amssymb}

\usepackage{caption}
\usepackage{subcaption}
\usepackage{graphicx}
\usepackage{wrapfig}

\usepackage{twemojis}

\newtheorem{theorem}{Proposition}

\newtheorem*{assumption*}{Assumptions}

\usepackage[normalem]{ulem}

\newcommand{\method}{SpLiCE~} 

\newcommand\blfootnote[1]{%
  \begingroup
  \renewcommand\thefootnote{}\footnotetext{#1}%
  \endgroup
}

\title{Interpreting CLIP with Sparse Linear \\Concept Embeddings (\texttt{SpLiCE} \texttwemoji{knot})}

\author{ 
 Usha Bhalla$^*$ \\
 Harvard University $^{a,b}$
   \And
 Alex Oesterling$^*$ \\
  Harvard University $^b$\\
  \And
 Suraj Srinivas \\
  Harvard University $^b$\\
  \AND 
 Flavio P. Calmon\footnotemark[2] \\
  Harvard University $^b$\\
    \And
 Himabindu Lakkaraju\footnotemark[2] \\
  Harvard University $^{b,c}$\\
}

\begin{document}

\maketitle

\blfootnote{$^*$ Equal contribution, order by coin flip. \qquad $^\dagger$ Equal contribution, alphabetic order.}
\blfootnote{$^a$ Kempner Institute for the Study of Natural \& Artificial Intelligence}
\blfootnote{$^b$ School of Engineering and Applied Sciences \qquad $^c$ Harvard Business School}

\begin{abstract}
CLIP embeddings have demonstrated remarkable performance across a wide range of multimodal applications. However, these high-dimensional, dense vector representations are not easily interpretable,  limiting our understanding of the rich structure of CLIP and its use in downstream applications that require transparency. 
In this work, we show that the semantic structure of CLIP's latent space can be leveraged to provide interpretability, allowing for the decomposition of representations into semantic concepts. 
We formulate this problem as one of sparse recovery and propose a novel method, Sparse Linear Concept Embeddings (\method \texttwemoji{knot}), for transforming CLIP representations into sparse linear combinations of human-interpretable concepts. Distinct from previous work, \method is task-agnostic and can be used, without training, 
to explain and even replace traditional dense CLIP representations, maintaining high downstream performance while significantly improving their interpretability. We also demonstrate significant use cases of \method representations including detecting spurious correlations and model editing. Code is provided at \url{https://github.com/AI4LIFE-GROUP/SpLiCE}.
\end{abstract}

\section{Introduction}

Natural images include complex semantic information, such as the objects they contain, the scenes they depict, the actions being performed, and the relationships between them. Machine learning models trained on visual data aim to encode this semantic information in their representations to perform a wide variety of downstream tasks, such as object classification, scene recognition, segmentation, or action prediction. However, it is often difficult to enforce explicit encoding of these semantics within model representations, and it is even harder to interpret these semantics post hoc to better understand what models may have learnt and how they leverage this information. Further, model representations can be brittle, encoding idiosyncratic patterns specific to individual datasets and modalities instead of general human-interpretable semantic information. Multimodal models have been proposed as a potential solution to this issue, and methods such as CLIP \citep{radford21alearning} have empirically been found to provide highly performant, semantically rich representations of image data. The richness of these representations is evident from their high performance on a variety of tasks, such as zero-shot classification and image retrieval \citep{radford21alearning}, image captioning \citep{mokady2021clipcap}, and image generation \citep{ramesh2022hierarchical}. However, despite their performance, it remains unclear how to quantify the semantic content contained in their dense representations. In this work, we answer the question: \textit{can we decompose CLIP embeddings into human-interpretable representations of the semantic concepts they encode?} This can provide insight into the types of tasks CLIP can solve, the biases it may contain, and the manner through which downstream predictions are made.

Existing literature in areas such as concept bottleneck models \citep{koh2017understanding}, disentangled representation learning \citep{bengio2013deep}, and mechanistic intepretability \citep{olah2017feature} have proposed various approaches to understanding the semantics encoded by representations. However, these methods generally require predefined sets of concepts \citep{kim2018interpretability}, data with concept labels \citep{hsu2023disentanglement}, or rely on qualitative visualizations, which can be unreliable \citep{geirhos2023don}. Similar to these lines of work, we aim to recover representations that reflect the underlying semantics of the inputs. However, distinct from these works, we propose to do this in a task-agnostic manner and without concept datasets, training, or qualitative analysis of visualizations.

Our method, \method\!, leverages the highly structured and multimodal nature of CLIP embeddings for interpretability, and decomposes CLIP representations via a semantic basis to yield a sparse, human-interpretable representation. Remarkably, these interpretable \method embeddings have favorable accuracy-interpretability tradeoffs when compared to black-box CLIP representations on metrics such as zero-shot accuracy. 
Our overall contributions are:

\begin{figure}[t]
    \centering
    \includegraphics[width=\linewidth]{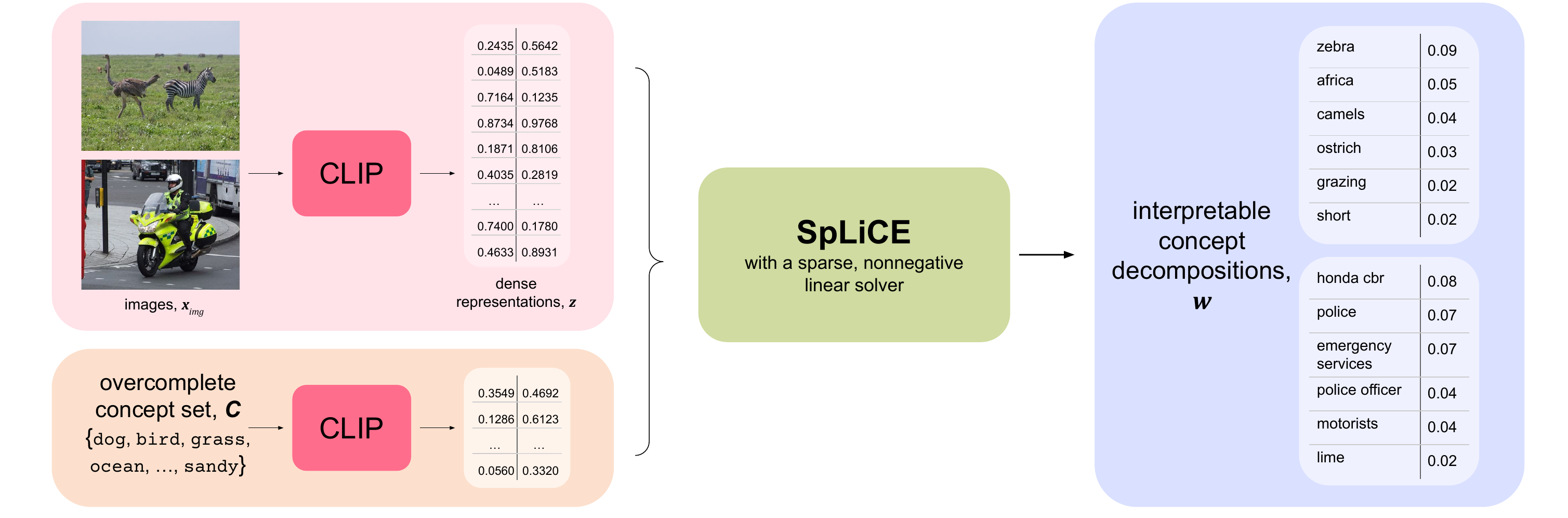}
    \caption{Visualization of \method\!, which converts dense, uninterpretable CLIP representations (\textbf{z}) into sparse semantic decompositions (\textbf{w}) by solving for a sparse nonnegative linear combination over an overcomplete concept set (\textbf{C}).}
    \label{fig:teaser}
\end{figure}

\begin{itemize}
    \item In Sections \ref{sec:assumptions} and \ref{sec:method}, we formalize the sufficient conditions under sparse decomposition of CLIP is feasible,
    and introduce \method\!, a novel method that decomposes dense CLIP embeddings into sparse, human-interpretable concept embeddings. 
    \item Our extensive experiments in Section \ref{sec:experiments} reveal that \method recovers highly sparse\footnote{we recommend and use sparsity levels of $\sim$ 10-30 in practice}, interpretable representations with high performance on downstream tasks, while accurately capturing the semantics of the underlying inputs. 
    \item In Section \ref{sec:case_studies}, we present two case studies for applying \method\!: spurious correlation detection, and model editing. Using \method\!, we uncover a spurious correlation in the CIFAR100 dataset, where we find the "woman" concept and the "swimwear" concept to be correlated owing to the prevalence of women in swimwear in CIFAR100.
\end{itemize}
\section{Related Work}

\textbf{Linear Representation Hypothesis.} In language modeling, the \textit{linear representation hypothesis} suggests that many semantic concepts are approximately linear functions of model representations \citep{mikolov2013linguistic, park2023linear, arora2018linear, arora2016latent, faruqui2015sparse}. Recent work has also shown that multimodal models encode concepts additively, behaving like bags-of-words representations \citep{yuksekgonul2022and}. Relatedly, \citep{merullo2022linearly, seth2023dear} show that there exists a linear mapping between image and text embeddings in arbitrary models. Our work makes use of these distinct but related observations to convert dense CLIP representations to sparse semantic ones. 

\begin{figure*}[!ht]
    \centering
        \includegraphics[width=\linewidth]{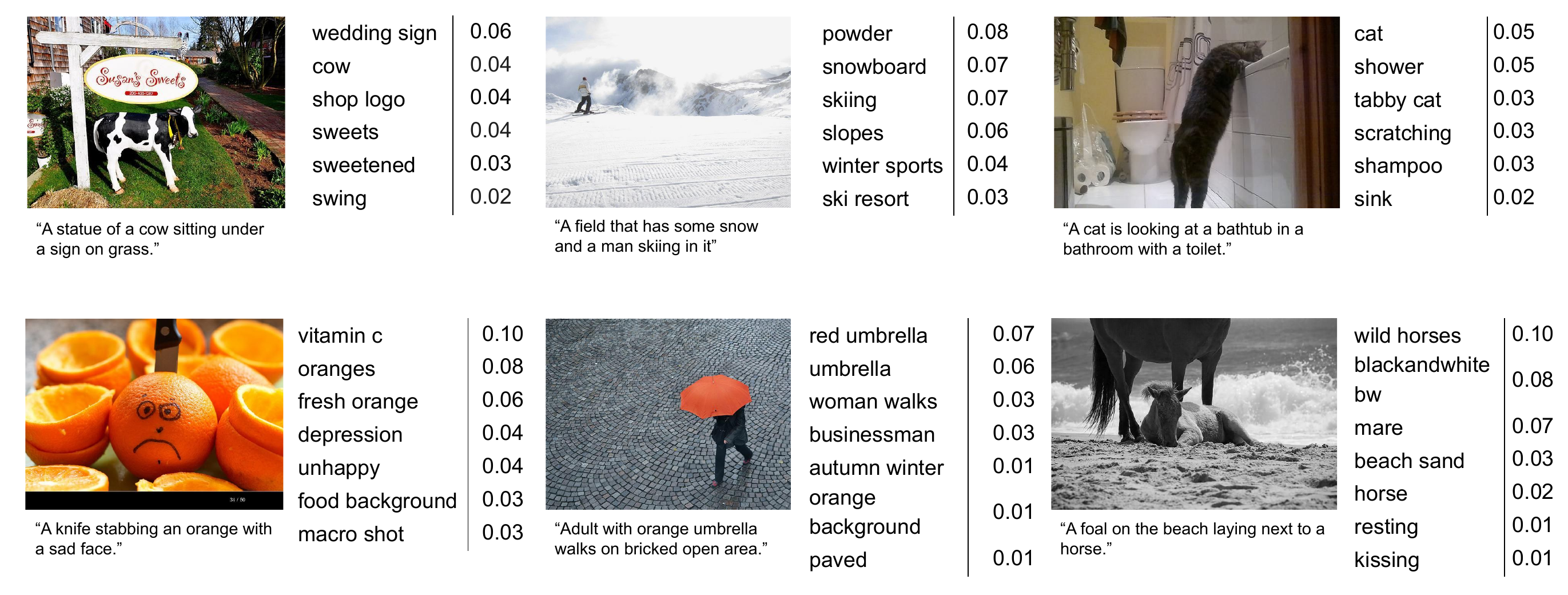}
        \caption{Example images from MSCOCO shown with their captions below and their concept decompositions on the right. We display the top seven concepts for visualization purposes, but images in the figure had decompositions with 7-20 concepts.}
        \label{fig:example_decomps}
\end{figure*}

\textbf{Concept Bottlenecks and Attribute Learning.} Concept Bottleneck Models (CBMs) \citep{koh2020concept}, and attribute-based models \citep{lampert2009learning, torresani2010efficient, kumar2009attribute} learn intermediate representations of scores over concepts or image attributes for use with a final linear classification head, creating interpretable concept representations. However, these require expert-labeled concept or attribute datasets to train, which is expensive. Recent work on concept-bottlenecks for multimodal models avoids needing such labeled datasets, but still requires concept labels for specific tasks, which is obtained by querying large language models (LLMs) \citep{chattopadhyay2023information, oikarinen2023label, panousis2023sparse}, making these methods task-specific and heavily reliant on the domain knowledge and subject to the biases of LLMs. On the other hand, \method uses a large-scale and overcomplete concept dictionary, avoiding dependence on training, specific domain knowledge, or a downstream task. Consequently, it can even be applied to understand unstructured, unsupervised image datasets in a label-free manner.

\textbf{Mechanistic Interpretability and Disentanglement.} 
Mechanistic interpretability explains representations through model activations, by labeling circuits and neurons in networks with feature visualization \citep{olah2017feature, olah2020zoom} or by measuring concept activations and directions in latent space \citep{bau2017network, fong2018net2vec, kim2018interpretability, mcgrath2022acquisition, lucieri2020interpretability, zhou2018interpretable}. Recent work \citep{fel2023holistic} combines these methods, using dictionary learning to extract visual concept activations, whose semantics can be identified via feature visualization.  Work in disentangled representation learning has developed architectures that capture independent factors of variation in data \citep{hsu2023disentanglement,bricken2023towards, chen2018isolating, bengio2013deep, comon1994independent, hyvarinen2000independent}, allowing for manual probing of disentangled representations for human-interpretable concepts. 
In both mechanistic interpretability and disentangled representation learning, methods typically rely on labeled concept sets, manual labeling of visualizations, or computationally intensive searches over data and latent representations or neurons to identify concepts. However, associating human-interpretable semantics with arbitrary neurons or latent directions is challenging, leading to the unreliability \citep{geirhos2023don, makelov2023subspace} exhibited by such methods. Our approach side-steps this issue by decomposing CLIP representations into a predetermined set of concepts.

\textbf{CLIP Interpretability.} Many recent works leverage the semantic structure of CLIP and its text encoder to interpret representations. For example, \citep{moayeri2023text}, \citep{yuksekgonul2022post}, and \citep{yun2022vision} construct concept similarity scores of image embeddings for use by downstream CBMs or probes, but these representations are not interpretable due to their lack of sparsity and the presence of negative concepts. \citet{chen2023stair} create a custom vision-language architecture with a sparse latent dictionary, but it requires training from scratch and cannot be used post-hoc to explain existing models. \citet{gandelsman2023interpreting} also leverage the text encoder of CLIP to explain components of the image embedding, but are limited to ViT architectures and take a mechanistic interpretability-style approach requiring a labeled text dataset. \citet{chattopadhyay2023information} build concept bottlenecks for specific classification tasks by expressing CLIP image representations as a sparse linear combination of task-specific concept vectors. However, their decomposition includes negative concepts, reducing interpretability, and uses task-specific concept dictionaries. \citet{grootendorst2022bertopic} generate textual topics of datasets through multimodal topic modeling, which cannot provide explanations of individual representations. Distinct from these works, \method is more interpretable due to its sparsity, overcompleteness, and non-negativity, and is task-agnostic, aiming to serve as a drop-in replacement for black-box CLIP representations without requiring training. 
\section{When do Sparse Decompositions Exist?}
\label{sec:assumptions}
In this section, we aim to answer the question: \textit{under what conditions can CLIP representations be decomposed into sparse semantic representations}? To do so, we must reason about both the properties of CLIP as well as the properties of the underlying data. 

\newcommand{\Z}{\mathbf{z}}
\newcommand{\E}{\mathbf{e}}
\newcommand{\X}{\mathbf{x}}
\newcommand{\C}{\mathbf{C}}
\newcommand{\R}{\mathbb{R}}
\newcommand{\W}{\mathbf{w}}

\textbf{Notation.} Let $\X^{\mathrm{img}} \in \R^{d_i}$, $\X^{\mathrm{txt}} \in \R^{d_t}$ be image and text data, respectively. Given the CLIP image encoder $f: \R^{d_i} \rightarrow \R^d$ and text encoder $g: \R^{d_t} \rightarrow \R^d$, we define CLIP representations in $\mathbb{R}^d$ as $\Z^{\mathrm{img}} = f(\X^{\mathrm{img}})$ and $\Z^{\mathrm{txt}} = g(\X^{\mathrm{txt}})$. Our method uses dictionary learning to approximate $\Z^{\mathrm{img}}$ with a concept decomposition $\W^* \in \mathbb{R}_+^c$ over a fixed concept vocabulary $\C \in \mathbb{R}^{d \times c}$.  We define the resulting reconstruction of $\Z^{\mathrm{img}}$ from $\C$ and $\W^*$ as $\hat{\Z}^{\mathrm{img}}$. 

The goal of our method is to approximate $f(\X^\text{img}) \approx \mathbf{C} \W^*$, such that $\W^*$ is non-negative and sparse, and in this section we formalize when this is possible. We begin by considering a data-generating process for coupled image and text samples. Specifically, we model the generative process parameterized by a $k$-dimensional latent concept vector ${\omega} \in \mathbb{R}_+^{k}$ and a random noise vector $\epsilon \in \R^{l}$ as
\begin{align*}
    \X^\text{img} = {h}^\text{img}(\omega, \epsilon), ~~~~ \X^\text{txt} = h^\text{txt}(\omega, \epsilon),
    ~~~\omega \sim \rho, 
    ~~~\epsilon \sim \phi,
\end{align*}
where $\rho$ is a prior distribution over semantic concepts, $\phi$ is a prior distribution over nonsemantic concepts (such as camera orientation and lighting for images or arbitrary choices between synonyms for text), and $h^\text{img}: \R^{k + l} \rightarrow \R^{d_i}$, and $h^\text{txt}: \R^{k + l} \rightarrow \R^{d_t}$ represent the real-world data-generating process from latent variables $(\omega, \epsilon)$ to images and text respectively. Here, each coordinate $\omega_i \in \R_+$ encodes the degree of prevalence of the $i^\text{th}$ concept in the underlying data. We now list a set of sufficient conditions for our data-generating process and CLIP that admit a sparse decomposition of images into concepts.

\textbf{Sufficient Conditions for Sparse Decomposition.}
\begin{enumerate} 
    \item Images and text are sparse in concept space: for some $\alpha \ll k$, we have $\| \omega \|_0 \leq \alpha, \forall~ \omega \sim \rho$.
    \item  CLIP captures semantic concepts $\omega$ and not $\epsilon$: $\forall \epsilon, \epsilon', f \circ h^\text{img}(\omega, \epsilon) = f \circ h^\text{img}(\omega, \epsilon')$ and similarly for $h^\text{txt}.$
    \item CLIP is linear in concept space: $g \circ h^\text{txt}$ and $f \circ h^\text{img}$ are linear in $\omega$.
    \item CLIP image and text encoders are aligned: for a given $\omega$, $f \circ h^\text{img} (\omega, \epsilon) = g \circ h^\text{txt} (\omega, \epsilon)$.
\end{enumerate}

 We emphasize that the goal of enumerating a set of sufficient conditions for sparse decomposition is not to claim that these exactly hold in practice, but rather to reason about when sparse decompositions--as done in this work--are appropriate. In the Appendix (Section \ref{sec:proof}, Prop.~\ref{thm:prop1}) we formalize and prove this claim, but in the interest of simplicity we keep the discussion here informal. 
We note that many of these are natural; Assumption 1 reflects how real-world images and text are simple and rarely contain complex semantic content, and the CLIP training process optimizes for Assumption 2 and 4\footnote{In practice we find that CLIP's image and text encoders are not fully aligned, so we apply a preprocessing step (Sec \ref{sec:modality_alignment}).}.
Of these, the most critical one is Assumption 3, which closely relates to the linear representation hypothesis \citep{park2023linear}, which we investigate below. 

\paragraph{Sanity Checking CLIP's Linearity.}
\label{sec:clip_lin}

\begin{table}[t]
\caption{Sanity checking the linearity of CLIP Embeddings.}
\label{tab:clip_lin}
\begin{center}
\begin{small}
\begin{sc}
\begin{tabular}{lcccr}
\toprule
           & $w_a$     & $w_b$     & cosine($\hat{z}, z$) \\ 
\midrule
ImageNet   & 0.48 $\pm$ 0.09 & 0.45 $\pm$ 0.09 & 0.76 $\pm$ 0.05       \\
CIFAR100   & 0.45 $\pm$ 0.08 & 0.42 $\pm$ 0.08 & 0.75 $\pm$ 0.03       \\
MIT States & 0.48 $\pm$ 0.09 & 0.45 $\pm$ 0.09 & 0.76 $\pm$ 0.05       \\
\midrule
COCO Text  & 0.59 $\pm$ 0.12 & 0.47 $\pm$ 0.12 & 0.88 $\pm$ 0.04       \\
\bottomrule
\end{tabular}
\end{sc}
\end{small}
\end{center}
\end{table}

We provide evidence for the third assumption, the linearity of CLIP, in a toy setting. We begin by asking the following question to confirm the general linearity of CLIP embeddings: ``if two inputs are concatenated, does their joint embedding equal the average of their two individual embeddings?". For the image domain, we combine two images, $x_a, x_b$, to form their composition $x_{ab}$ by placing $x_a$ in the top left quarter and $x_b$ in the bottom right quarter of a blank image. For the text domain, we simply append text $x_b$ to text $x_a$ to form $x_{ab}$. We then embed $x_a, x_b, x_{ab}$ with CLIP to get $z_a, z_b, z_{ab}$. Solving the equation $w_a*z_a + w_b*z_b = z_{ab}$ for scalar weights $w_a, w_b$ then allows us to assess the linearity of $z_a, z_b, z_{ab}$. We report $w_a, w_b$ and the cosine similarity between $\hat{z}_{ab} = [z_a, z_b] \cdot [w_a, w_b]$ and $z_{ab}$ in Table \ref{tab:clip_lin}. 

In general, we find that the composition of two inputs results in an embedding that is approximately equal to the average of the two input components, with $w_a, w_b$ being very close to 0.5 across all datasets and for both modalities, providing preliminary evidence for the linearity of CLIP embeddings for both image and language. 
\color{black}
\section{Method}
\label{sec:method}

In this section, we introduce \method \!, a method for expressing CLIP's image representations as sparse, nonnegative, linear combinations of concept dictionary elements. We begin by framing this problem as one of sparse recovery. We then discuss our design choices, including how we choose the concept dictionary and how to address the modality gap between CLIP's images and text representations. Finally, we formalize the optimization problem used in this work.

\subsection{Sparse Nonnegative Concept Decomposition} \label{sec:splice}

Our goal is to construct decompositions of dense CLIP representations that are human-interpretable, useful, and faithful. To do so, we formulate decomposition as a sparse recovery problem with three main desiderata. First, for the decompositions to be interpretable to humans they must be comprised of human interpretable atoms. We argue that language is a naturally interpretable interface for humans, and construct our concept vocabulary $\mathbf{C}$ out of 1- and 2-word atoms, such as ``coffee'', ``silver'', and ``birthday party''. 
Second, our decompositions must be simple and concise, which can be formulated as a sparsity constraint on the recovery. 
A large body of work in computational linguistics \citep{murphy2012learning, fyshe2014interpretable, fyshe2015compositional, faruqui2015sparse}, neuroscience \citep{olshausen1997sparse, olshausen1996emergence}, and interpretability \citep{ramaswamy2022overlooked, ribeiro2016should, zhou2018interpretable} have demonstrated that a human-aligned semantic model should be sparse in representation. 
Furthermore, \citep{ramaswamy2022overlooked} found that users can best understand explanations with fewer than 32 concepts while in linguistics, \citep{vinson2008semantic, mcrae2005semantic, garrard2001prototypicality} find participants describe concepts and objects with up to 20 semantic properties, motivating our desiderata of sparsity. 
Third, our decompositions must be constructive, i.e., we must decompose representations in terms of their constituent concepts. For this reason, we require the weights of decompositions to be strictly nonnegative, to avoid having ``negative'' concept weights which do not always carry semantic meaning. Furthermore, prior work by \citet{zhou2018interpretable} has argued that \textit{``negations of concepts are not as interpretable as positive concepts.''}  More specifically, while a small set of concepts have well-defined antonyms which may be viewed as their negative counterparts (\texttt{``day''} $\leftrightarrow$ \texttt{``night''}), negative concepts do not carry semantic meaning in general (\texttt{``tiger''} $\leftrightarrow$ \texttt{??}). Furthermore, we find that even when antonyms exist, they are not negatives of each other in CLIP latent space (see Appendix \ref{sec:app_negweights}).
To avoid dependence on negative weights and ensure that all concepts are captured, we construct an overcomplete dictionary containing a wide range of concepts, including antonyms. We build on top of this literature and provide a semantic decomposition satisfying these properties suitable for multimodal models like CLIP.

\begin{figure*}
\centering
\begin{subfigure}{0.32\linewidth}
   \includegraphics[width=\linewidth, trim={1ex 0ex 10ex 1ex}, clip]{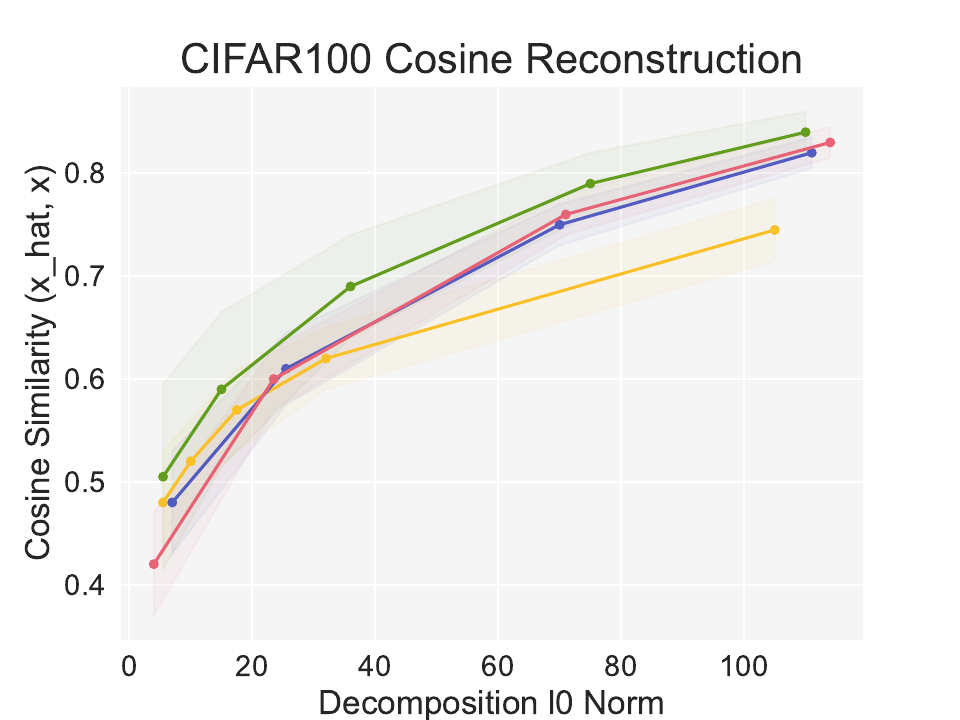}
\end{subfigure}
\hfill
\begin{subfigure}{0.32\linewidth}
   \includegraphics[width=\linewidth, trim={1ex 0ex 10ex 1ex}, clip]{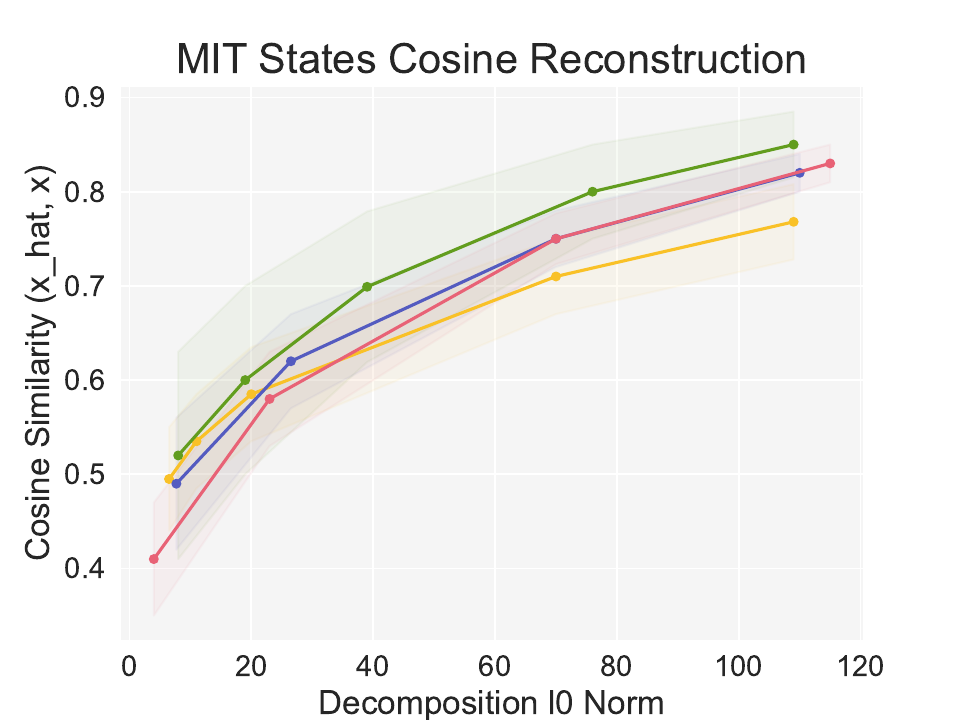}
\end{subfigure}
\hfill
\begin{subfigure}{0.32\linewidth}
   \includegraphics[width=\linewidth, trim={1ex 0ex 10ex 1ex}, clip]{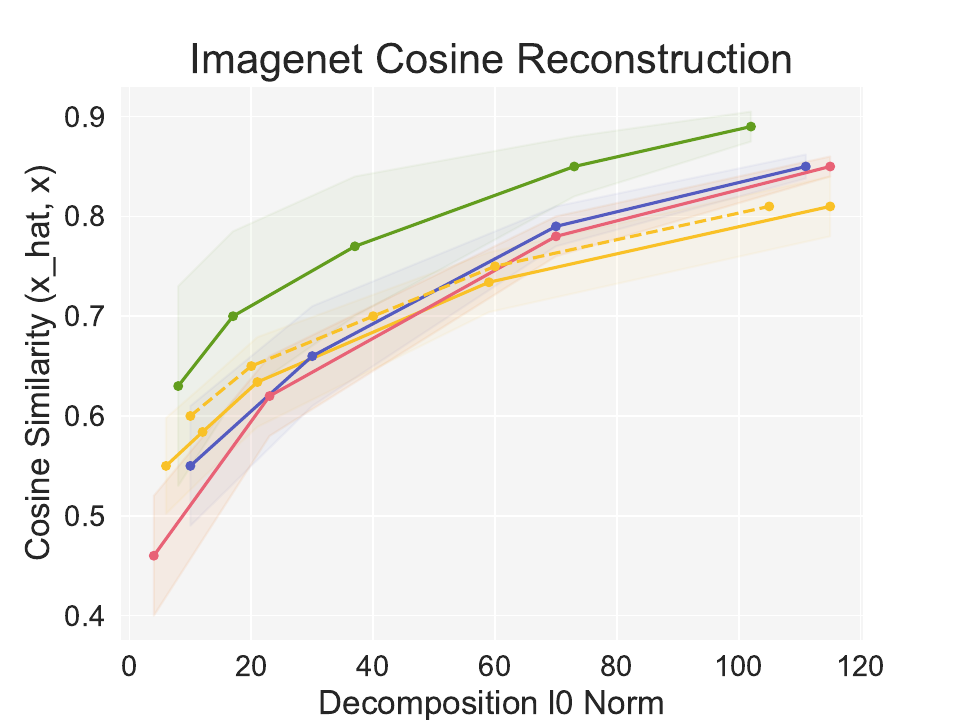}
\end{subfigure}
\begin{subfigure}{0.32\linewidth}
   \includegraphics[width=\linewidth, trim={1ex 0ex 10ex 1ex}, clip]{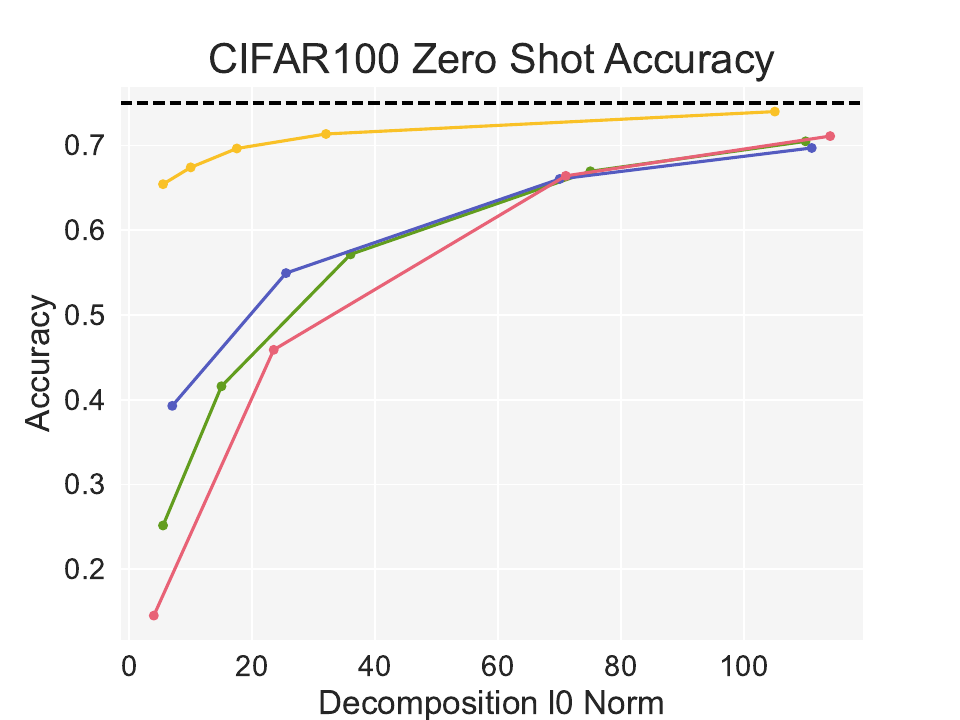}
\end{subfigure}
\hfill
\begin{subfigure}{0.32\linewidth}
   \includegraphics[width=\linewidth, trim={1ex 0ex 10ex 1ex}, clip]{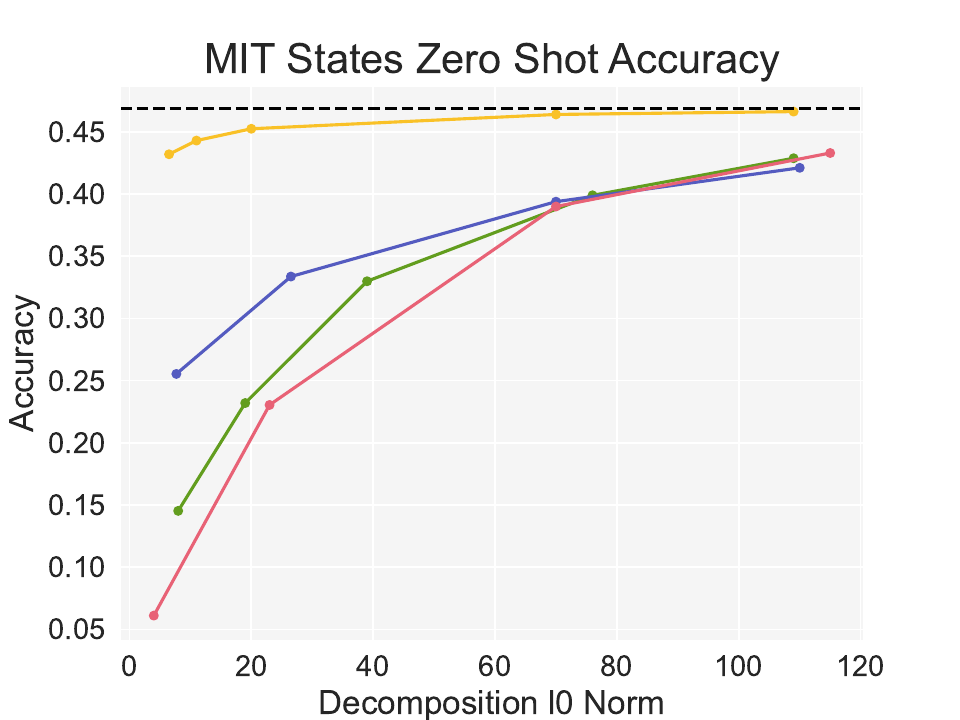}
\end{subfigure}
\hfill
\begin{subfigure}{0.32\linewidth}
   \includegraphics[width=\linewidth, trim={1ex 0ex 10ex 1ex}, clip]{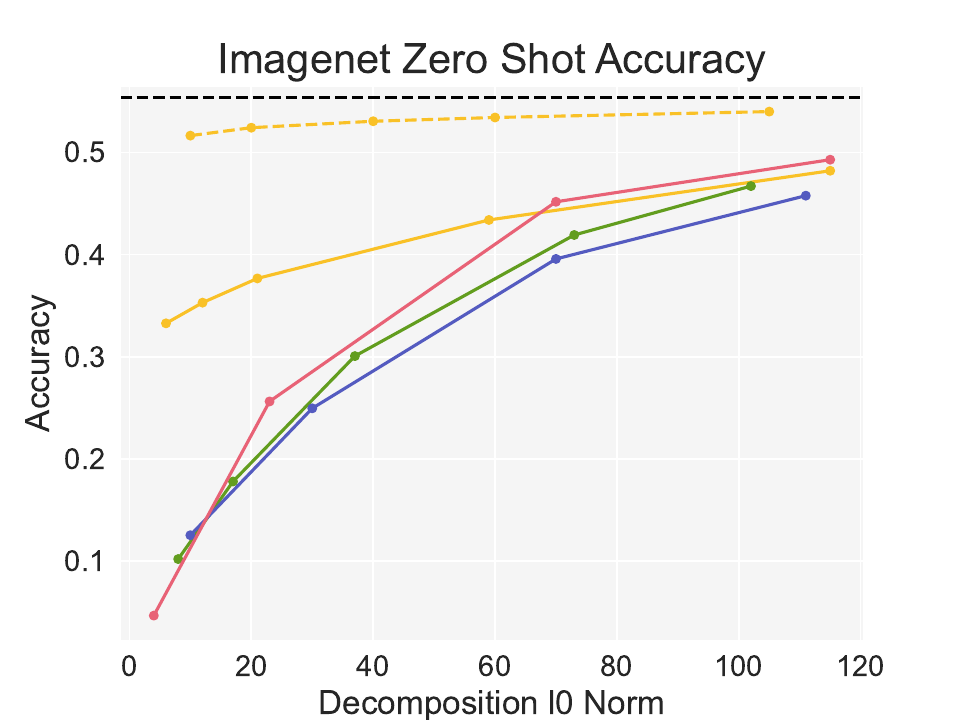}
\end{subfigure}
\begin{subfigure}{0.9\linewidth}
   \includegraphics[width=\linewidth]{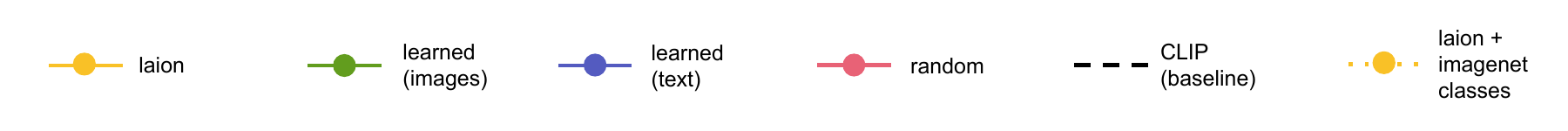}
\end{subfigure}
\caption{Performance of \method decomposition representations on zero-shot classification tasks (bottom row) and cosine similarity between CLIP embeddings and \method embeddings (top row).
Our proposed semantic dictionary (yellow) closely approximates CLIP on zero-shot classification accuracy, but not on the cosine similarity. This indicates that \method captures the semantic information in CLIP, but not its non-semantic components, explaining both the high zero-shot accuracy and low cosine similarity. See \S \ref{sec:perf} for discussion.
}
\label{fig:cosine}
\end{figure*}

\textbf{Concept Vocabulary.}
\label{sec:concept_vocab}
Natural language is an intuitive, interpretable, and compact medium for communicating semantic information. Thus, we choose to represent the semantic content contained in CLIP embeddings as combinations of natural language semantic concepts, where we define concepts as semantic units that can be expressed concisely, by one- or two-word phrases. Given that CLIP is used in a wide variety of downstream applications and is trained without a specific task in mind, we want our concept dictionary to be task-agnostic and to span \textit{all possible concepts CLIP could have learnt}.
To construct this vocabulary, we consider the most frequent one- and two-word bigrams in the text captions of the LAION-400m dataset \citep{schuhmann2021laion}, the dataset that most CLIP variants are trained on. We filter the captions to remove any NSFW samples and prune our concept set such that no two concept embeddings have a cosine similarity greater than 0.9. We also remove bigrams highly similar ($>0.9$ cosine similarity) to the average of their individual words. We finally choose the top $10,000$ most common single-word concepts and the top $5000$ most common two-word concepts as our concept vocabulary. We note that this vocabulary offers distinct advantages over those used in prior works. In particular, it is \textit{task-agnostic}, meaning that the efficacy of the decomposition is (in principle) independent of individual datasets. Furthermore, this dataset imposes minimal priors from outside curators, such as human experts or LLMs \citep{chattopadhyay2023information, oikarinen2023label, panousis2023sparse}. This allows us to interpret data through the lens of CLIP, to understand the information encoded, including potential biases and mistakes. 

\textbf{Modality Alignment.}\label{sec:modality_alignment} In order to decompose images into nonnegative combinations of text concepts, we must ensure that our concept set spans the space of possible image embeddings. 
However, \citep{liang2022mind} show the existence of a modality gap in CLIP, where image and text embeddings can lie in non-identical spaces on the unit sphere. 
We empirically find that CLIP image and text embeddings exist on two cones, as the distribution of pairwise cosine similarities between pairs of MSCOCO images and pairs of MSCOCO text captions concentrate at positive values, whereas the distribution of pairwise cosine similarities across modalities concentrates closer to zero. (See Appendix Fig. \ref{fig:modality_alignment}).
Not only does this prevent nonnegative decomposition, it also violates Assumption 4 from Section \ref{sec:assumptions}. 
To rectify this, we mean-center CLIP images with the image cone mean, estimated over MSCOCO ($\mathbf{\mu_{img}}$), and compute decompositions over the mean-centered concept vocabulary ($\mathbf{\mu_{con}}$). Note that the embeddings need to be re-normalized after centering to ensure they lie on the unit sphere. To convert our decompositions back into dense representations ($\hat{\Z}^\text{img}$), we uncenter the normalized dense embeddings $\hat{\Z}^\text{img}$ by adding the image mean back in and normalizing once again, to ensure they lie on the same cone as the original CLIP embeddings ($\Z^\text{img}$). 

\textbf{Optimization Problem.} Our optimization problem is formulated as follows. Let $\sigma(\X) = \X / \| \X \|_2$ be the normalization operation. Given a set of semantic concepts $\X^{\textrm{con}}=~$[``\texttt{dog}'', ``\texttt{tabby cat}'', ``\texttt{cloudy}'', $\cdots$ ], we construct a centered vocabulary $\C = \left[\sigma(g(\X^{\textrm{con}}_1) - \mathbf{\mu_{\textrm{con}}}), \cdots, \sigma(g(\X^{\textrm{con}}_c) - \mathbf{\mu_{\textrm{con}}})\right]$, where we recall that $g(\cdot)$ is the CLIP text encoder. Now, given the dictionary $\mathbf{C}$ and a centered  CLIP embedding $\Z = \sigma(\Z^{\textrm{img}}-\mathbf{\mu_{\textrm{img}}})$,  we seek to find the sparsest solution that gives us a cosine similarity score of at least $1-\epsilon$ for some small $\epsilon$:
\begin{align}
    \min_{\W \in \mathbb{R}_+^c} \|\W\|_0 
    ~~\text{s.t.}~~ \langle \Z, \sigma(\C\W) \rangle \geq 1-\epsilon. \label{eq:full_optimization}
\end{align}
As is standard practice, we relax the $\ell_0$ constraint and reformulate this as a minimization of MSE with an $\ell_1$ penalty, to construct the following convex relaxation\footnote{For more discussion on the relationship between Eq.~\eqref{eq:full_optimization} and Eq.~\eqref{eq:optimization}, see Appendix, Sec. \ref{sec:optimization_discussion}} of Eq.~\eqref{eq:full_optimization}:
\begin{align}
    \min_{\W \in \mathbb{R}_+^c} \|\C\W - \Z\|_2^2 + 2 \lambda \|\W\|_1. \label{eq:optimization}
\end{align}
Given the solution to the above problem $\W^*$, our reconstructed embedding is:
$    \hat{\Z}^{\textrm{img}} = \sigma(\C \W^* + \mathbf{\mu_{\textrm{img}}}).$ 
\section{Experiments}
\label{sec:experiments}

In this section, we evaluate our method to ensure that \method decompositions are interpretable, performant, and accurately reflect the semantic content of representations.

\subsection{Setup}
\label{sec:setup}
\textbf{Models. } 
All experiments shown in the main paper are done with the OpenCLIP ViT-B/32 model \cite{ilharco_gabriel_2021_5143773} with results for an additional model in Appendix \ref{sec:app_clip_rn50}. For all zero-shot classification tasks, we use the prompt template ``A photo of a \{\}''. \textbf{Datasets.}
We use CIFAR100 \cite{krizhevsky2009learning}, MIT States \cite{StatesAndTransformations}, CelebA \cite{liu2015faceattributes}, MSCOCO \cite{lin2014microsoft}, and ImageNetVal \cite{imagenet_cvpr09}
for our experiments with results for additional datasets in the Appendix (Section \ref{sec:add_zs})

\textbf{Decomposition. }
For all experiments involving concept decomposition, we use sklearn's \cite{scikit-learn} Lasso solver with a non-negativity flag and an $l_1$ penalty that results in solutions with $l_0$ norms of 5-20 (around 0.2-0.3 for most datasets).
We use a concept vocabulary chosen from a subset of LAION tokens as described in Section \ref{sec:concept_vocab}. Both image embeddings and dictionary concepts are centered and normalized as mentioned in Section \ref{sec:modality_alignment}, with the image mean used for centering computed over the MSCOCO train set and the concept mean computed over our chosen vocabulary. 

\subsection{Sparsity-Performance Tradeoffs}\label{sec:perf}
We assess the performance of \method decompositions by evaluating the reconstruction error in terms of cosine similarity between \method representations and CLIP embeddings, the zero-shot performance of \method \! decompositions, and the retrieval performance of \method embeddings. We compare the performance of decompositions generated from our semantic concept vocabulary to decompositions over random vocabulary and learned dictionary vocabulary baselines. All vocabularies are of size 15,000 concepts. The random vocabulary is sampled from a 512-dimensional normalized Gaussian distribution. The learned vocabularies are generated by using the Fast Iterative Shrinkage-Thresholding Algorithm (FISTA) \cite{beck2009fast} to learn optimal dictionaries given our sparse recovery problem (optimizing Equation \eqref{eq:full_optimization} for both $\mathbf{C}$ and $w$). Note that we learn separate dictionaries $\mathbf{C}_\text{img}$ and $\mathbf{C}_\text{text}$ to reconstruct MSCOCO image and text embeddings respectively. 
In Figure \ref{fig:cosine}, we plot the cosine reconstruction and zero-shot accuracy of image decompositions with the various dictionaries. We evaluate probing performance (Tables \ref{tab:decomp_probe_cifar}, \ref{tab:decomp_probe_mit}) and text-to-image and image-to-text retrieval in the Appendix (Figure \ref{sec:app_retrieval}). 

These results overall show \method efficiently navigates the interpretability-accuracy Pareto frontier and retains much of the performance of black-box CLIP representations with the semantic, human-interpretable LAION dictionary, significantly outperforming other dictionaries on semantic tasks such as zero-shot classification, probing, and retrieval. At the same time, we find that our semantic LAION dictionary does not result in accurate cosine similarity reconstructions of the original CLIP, often being on par with using random dictionaries. We believe this is because CLIP encodes both semantics of the underlying image and non-semantic "noise", which violates Assumption \#2 in Section  \ref{sec:assumptions}. Given that our \method decompositions only aim to encode semantics, they are unable to encode non-semantic aspects in the underlying representation, thus causing poor alignment in the cosine similarity sense, while simultaneously exhibiting excellent alignment on semantic tasks such as zero-shot accuracy.  For ImageNet, we find that many classes are animal species that cannot easily be described by 1-2 words (e.g. `\texttt{red-breasted merganser}', `\texttt{American Staffordshire terrier}'). Adding these class labels to our concept dictionary increases performance significantly, as shown by the dotted yellow line in Figure \ref{fig:cosine}.

\begin{figure}[h!]
    \centering
    \begin{minipage}{.5\columnwidth}
    \begin{subfigure}{0.49\columnwidth}
        \includegraphics[width=\columnwidth]{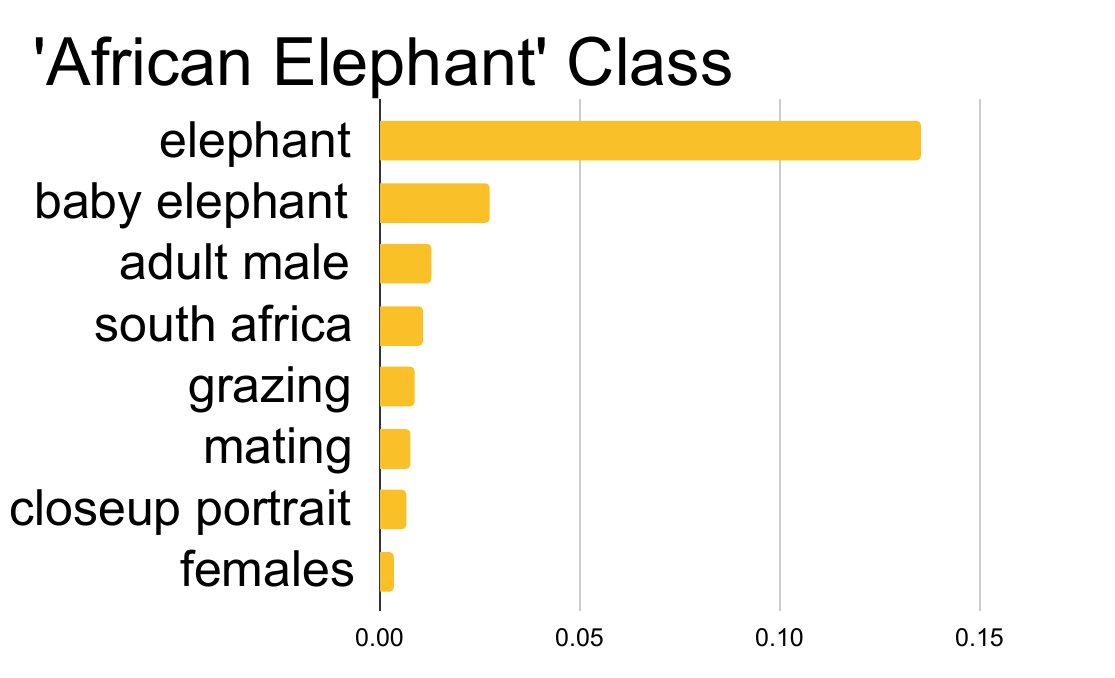}
    \end{subfigure}
    \begin{subfigure}{0.49\columnwidth}
        \includegraphics[width=\columnwidth]{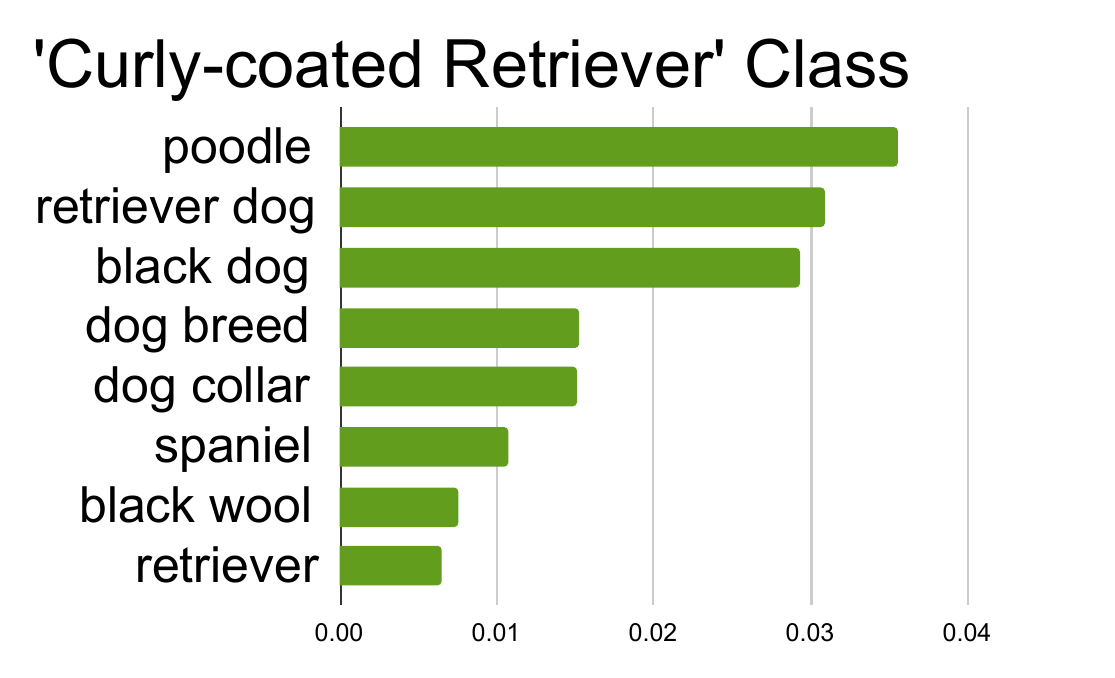}
    \end{subfigure}
    \\
    \begin{subfigure}{0.49\columnwidth}
        \includegraphics[width=\columnwidth]{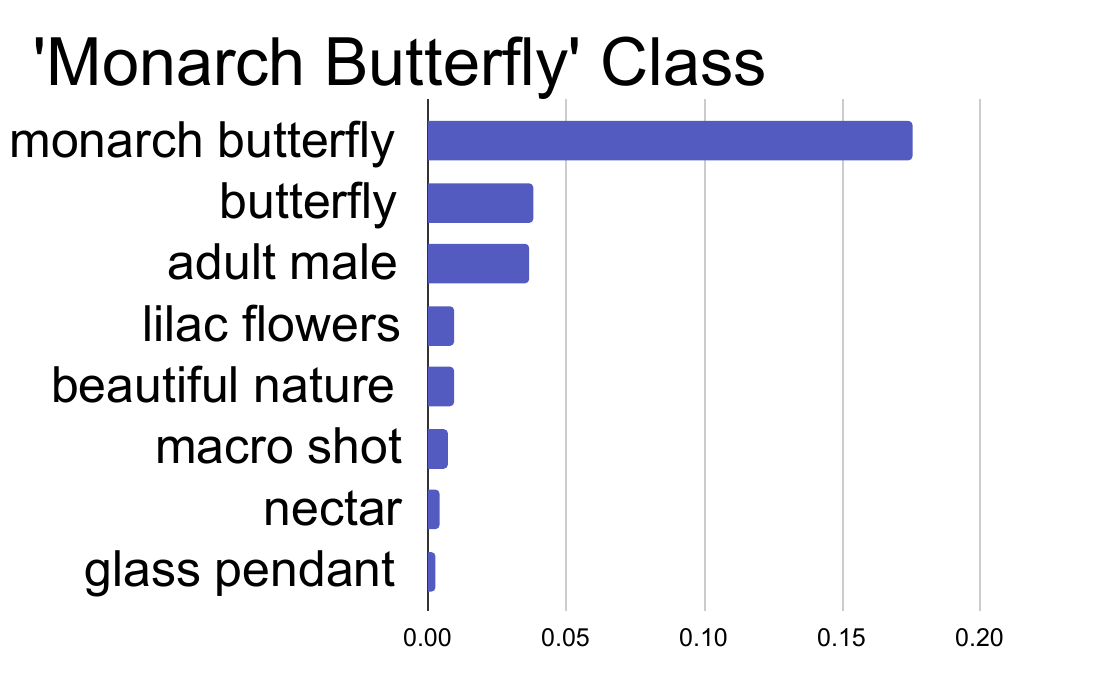}
    \end{subfigure}
    \begin{subfigure}{0.49\columnwidth}
        \includegraphics[width=\columnwidth]{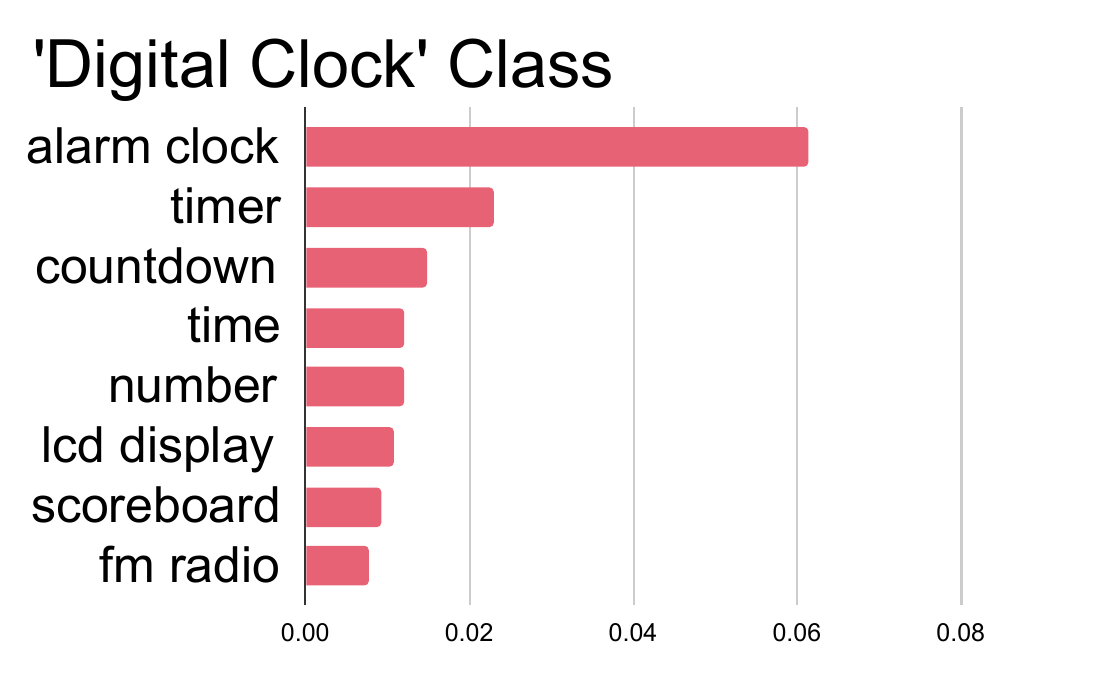}
    \end{subfigure}
    \end{minipage}%
    \begin{minipage}{.45\columnwidth}
    \centering
    \begin{subfigure}{\columnwidth}
        \includegraphics[width=\columnwidth]{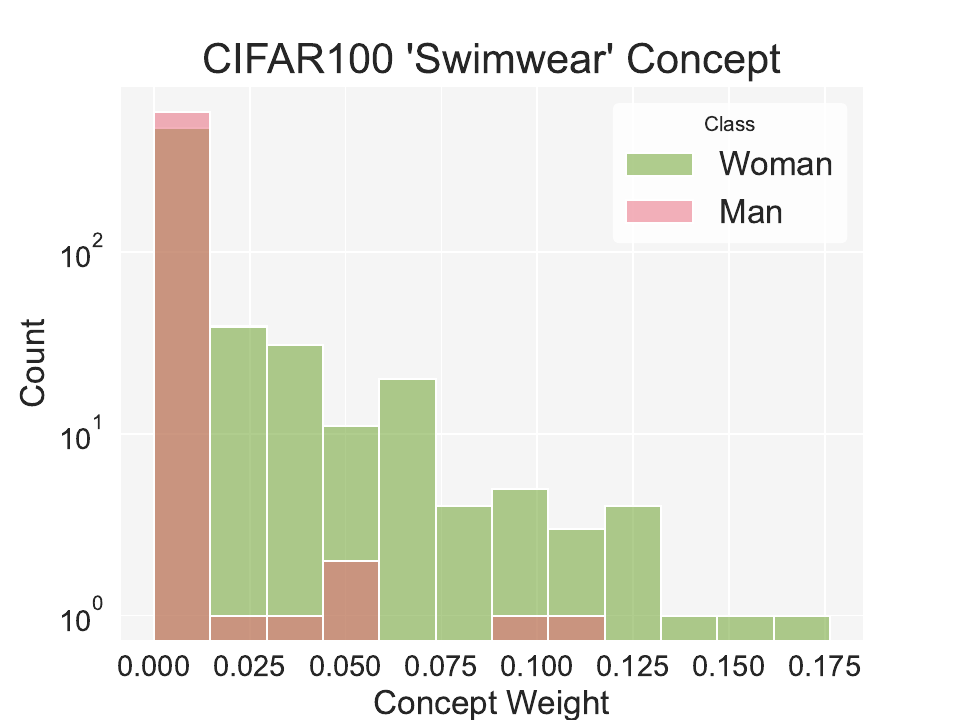}
    \end{subfigure}
    \end{minipage}
    \caption{\textbf{Left:} \method decompositions of ImageNet `\texttt{African Elephant}', `\texttt{Curly-coated Retriever}', `\texttt{Monarch Butterfly}', `\texttt{Digital Clock}' classes. \textbf{Right:} Distribution of ``\texttt{Swimwear}" concept in `\texttt{Woman}' and `\texttt{Man}' classes 
    of CIFAR100. }
    \label{fig:class_decomps_cifar_bias}

\end{figure}

\subsection{Ablation Studies}\label{sec:ablations}

We perform ablation studies to evaluate the effectiveness of the design decisions of \method\!, including the choice of vocabulary, the nonnegativity of the decompositions, and the modality alignment by ablating each choice and observing the effect on three metrics: zero-shot accuracy on CIFAR100, cosine similarity between the reconstructed and original embeddings of CIFAR100, and semantic relevance on MSCOCO. The first two metrics are the same as those presented in Figure \ref{fig:cosine}. We compute semantic relevance by tokenizing and filtering stop-words from the MSCOCO human-generated captions and embedding each token with CLIP. Then, we take all non-zero concepts output by \method and compute the Hausdorff distance between the sets of \method concepts and caption token embeddings. This essentially measures how aligned decompositions are with human captions. 
We observe that replacing our dictionary with the LLM-generated concept dictionary used by \cite{chattopadhyay2023information, panousis2023sparse, oikarinen2023label} significantly worsens the decomposition in terms of zero shot accuracy and cosine reconstruction.
While allowing for negative concept weights improves cosine reconstruction marginally, it decreases the semantic relevance of the decompositions, as negative concepts frequently correspond to concepts not present in images, and as such, are unlikely to be represented by human captions.
Finally, we see that modality alignment is necessary across all three metrics. Overall, these ablation studies show that each aspect of \method is necessary for creating human-interpretable, semantically relevant and highly performant decompositions.

\begin{figure*}
\centering
\begin{subfigure}{0.32\linewidth}
   \includegraphics[width=\linewidth, trim={1ex 0ex 10ex 1ex}, clip]{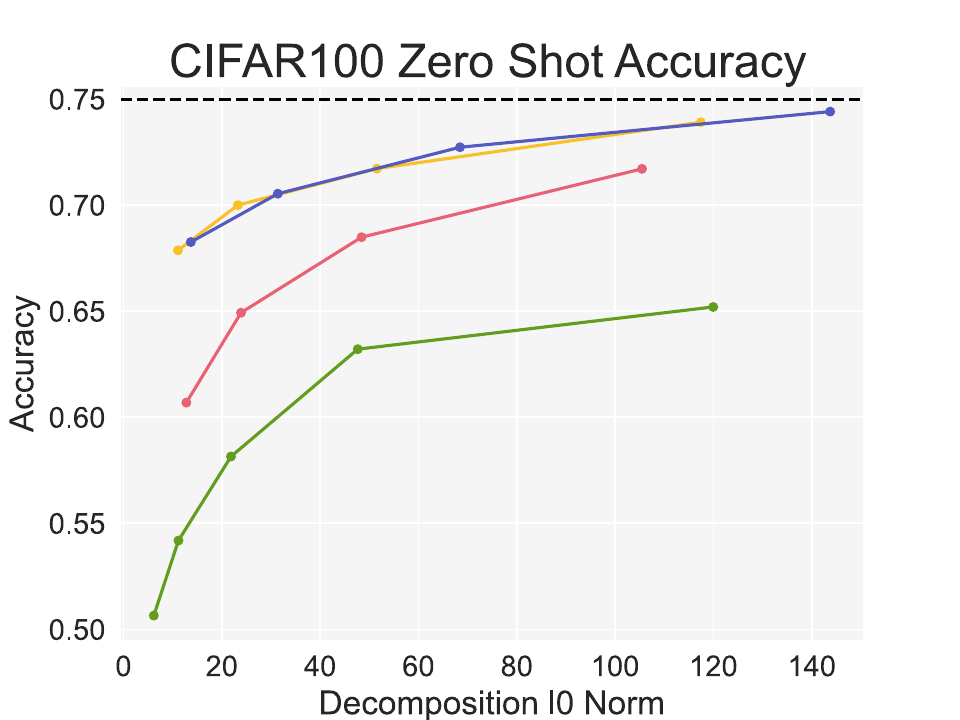}
\end{subfigure}\hfill
\begin{subfigure}{0.32\linewidth}
   \includegraphics[width=\linewidth, trim={1ex 0ex 10ex 1ex}, clip]{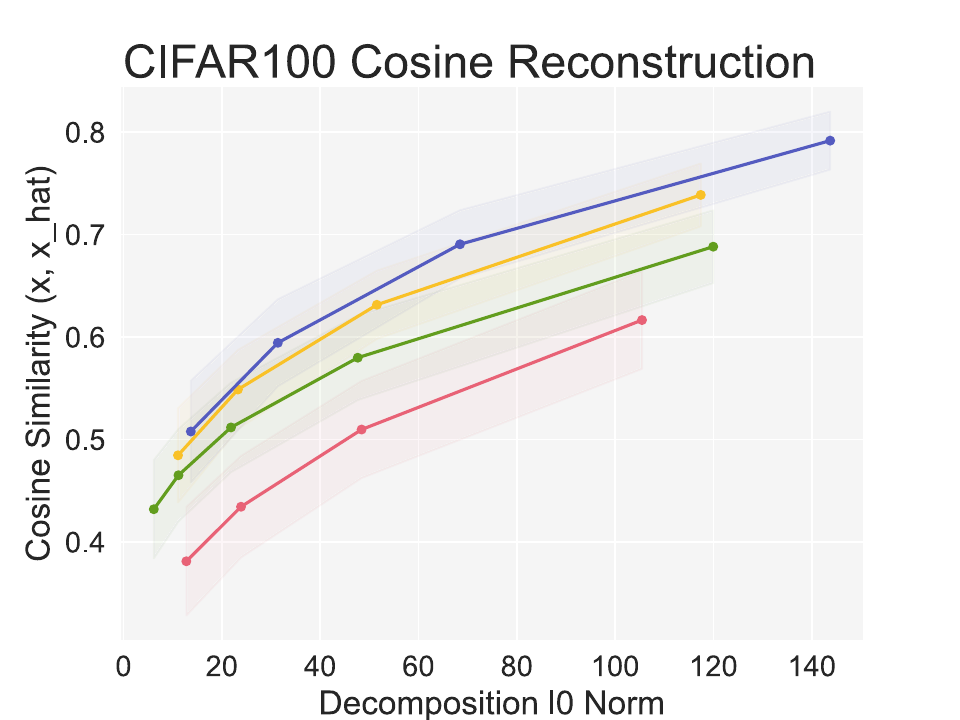}
\end{subfigure}\hfill
\begin{subfigure}{0.32\linewidth}
   \includegraphics[width=\linewidth, trim={1ex 0ex 10ex 1ex}, clip]{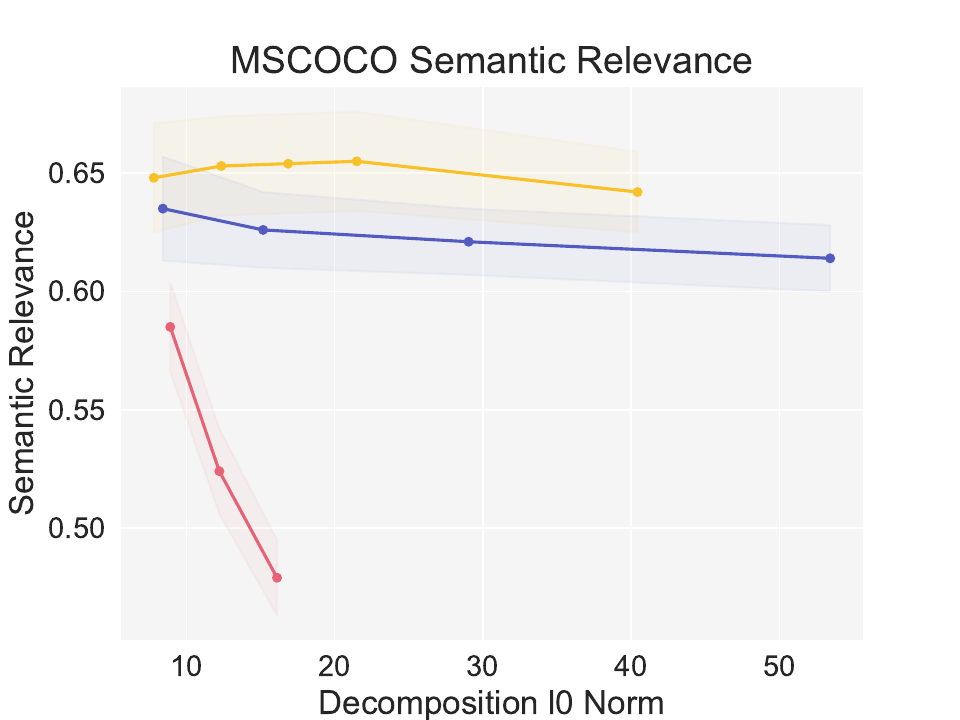}
\end{subfigure}
\begin{subfigure}{0.9\linewidth}
   \includegraphics[width=\linewidth]{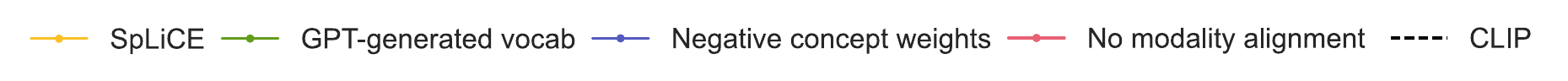}
\end{subfigure}
\caption{Ablation study evaluating the efficacy of \method design choices across three metrics: Zero-shot accuracy, cosine reconstruction, and semantic relevance of recovered tags. We find that all of our design choices, namely non-negativity, modality alignment, and usage of large task-agnostic dictionary are essential to performance. See \S \ref{sec:ablations} for discussion. 
}
\label{fig:ablations}
\end{figure*}

\subsection{Qualitative Assessment of Decompositions}
\textbf{Concept Decompositions for Images.} We visualize \method decompositions to qualitatively assess the semantic content of the images they represent. In Figure \ref{fig:example_decomps} we provide six sample decompositions from MSCOCO with their corresponding captions. We display the top seven concepts for each image and find that they generally well describe the semantics of the images. We also find that these qualitative examples yield interesting and unexpected insights into both CLIP and the data. In the top left image, we see that the decomposition includes the text present on the sign in the image, revealing that CLIP prioritizes text in images over objects. For the bottom left image, the decomposition correctly includes the concept ``\texttt{macro shot}'', revealing that CLIP encodes information regarding geometric perspective. The bottom right decomposition similarly features the concept ``\texttt{blackandwhite bw}'', indicating that CLIP encodes not only the objects present in images but also information about the lighting and color. Overall, these results suggest that \method may also be used as a zero-shot image tagging method to understand images.

\textbf{Concept Histograms for Datasets.}
Beyond concept-based explanations of individual images, we propose that SpLiCE can be used to better understand and summarize collections of images, such as entire datasets. To compute concept decompositions of sets of images, we decompose each individual image and aggregate the results, which we use to generate concept histograms of the dataset. We visualize four concept histograms for the ImageNet classes `\texttt{African Elephant}', `\texttt{Curly-coated Retriever}', `\texttt{Monarch Butterfly}', and `\texttt{Digital Clock}', in Figure \ref{fig:class_decomps_cifar_bias}. These decompositions provide information about the distribution of the data and how CLIP represents it. For example, digital clocks are differentiated from analog clocks through the concepts ``\texttt{lcd display}'' and ``\texttt{countdown}''. Monarch butterflies are highly correlated with the concept ``\texttt{lilac flowers}" in ImageNet, which we validated through manual inspection (nearly half of the monarch butterfly images in the validation set feature purple flowers). Interestingly, `\texttt{Curly-coated retrievers}' are represented as combinations of ``\texttt{poodle}'', ``\texttt{retriever dog}'', and ``\texttt{black dog}'', which perfectly describe the main characteristics of them: black retrievers with poodle-textured fur. 

\section{Case Studies and Applications of \method}
\label{sec:case_studies}

In this section, we present two example case studies using \method\!: (1) spurious correlation and bias detection in datasets and (2) debiasing classification models. We present additional case studies for (1) and (2), as well as (3) monitoring distribution shift in Appendix \ref{sec:app_spurious_correlations}, \ref{sec:app_wb_interv}, \ref{sec:app_dist_shift1} \ref{sec:app_dist_shift2}. We also present results from a user study to evaluate the human interpretability of \method in Appendix \ref{sec:user_study}, where we find that users prefer explanations generated by SpLiCE over existing Concept Bottleneck Model-based methods. 

\label{sec:spurious_detection}
\paragraph{Discovering Spurious Correlations in CIFAR100.}
Existing methods to detect spurious correlations in datasets generally require subgroup and attribute labels or rely on manual human inspection of images (see \cite{wu2023discover} for an overview), making it hard to scale to large datasets. \method\!, on the other hand, allows for fast automatic detection of such biases, without any labels, training, or even a task. To illustrate this, we study two classes of CIFAR100: `man' and `woman', in Figure \ref{fig:class_decomps_cifar_bias}. Upon decomposing these classes, we found that \{``\texttt{bra}'',``\texttt{swimwear}'' \} were two of the top ten most common concepts in the `woman' class. On the other hand, the only clothing-related concepts that appear in the top 50 most activated concepts for `man' are \{``\texttt{uniform}'', ``\texttt{tuxedo}'', ``\texttt{apparel}''\}. We visualize a histogram of the concept weights on swimwear- and undergarment-related concepts \{``\texttt{swimwear}'', ``\texttt{bra}'', ``\texttt{trunks}'', ``\texttt{underwear}''\} across both the train and test sets, and find that these concepts are much more likely to be activated for women than men. Manual inspection of CIFAR100 verifies the trend highlighted by SpLiCE, where \textit{at least 70 of the 600 images in the `woman' class feature women in bikinis, underclothes, or even partially undressed}, revealing stereotype bias in this popular dataset. We provide a similar study of the concept ``\texttt{desert}'' with respect to the  `camel' and `kangaroo' classes in CIFAR100 in Appendix \ref{sec:app_spurious_correlations}.

\label{sec:model_editing}

\paragraph{Model Editing on CelebA Attribute Classifiers.} 

\begin{table}[!h]
    \centering
      \caption{\small Evaluation of intervention on the concept `Glasses' for the CelebA dataset. SpLiCE allows for surgical removal of information related to whether or not someone is wearing glasses, without impacting other features such as gender. (ZS = Zero Shot Accuracy) }
        \label{tab:interv_glasses}
      \begin{small}
        \begin{sc}
        \begin{tabular}{lcccr}
            \toprule
                                & Gender & Glasses \\ 
            \midrule
            ZS CLIP             &  0.98  &  0.91  \\
            ZS SpLiCE           &  0.97  &  0.88  \\
            ZS Intervention SpLiCE &  0.96  &  \textbf{0.69}  \\
            \midrule
            Linear Probe        &  0.89  & 0.88   \\
            Intervention Probe  &  0.85  & \textbf{0.59}   \\
            \bottomrule
        \end{tabular}
        \end{sc}
        \end{small}
\end{table}

Concept-based representations unlock a key application: being able to intervene on and edit models. This edit can be performed in two equivalent ways: either on the concept representations themselves, where we can zero out a concept or on linear probes built upon the decompositions, where we can edit the weight matrix between concepts and class labels (similar to concept bottleneck models \cite{koh2020concept}). Here, we evaluate the efficacy of \method for these forms of model editing. Specifically, we consider two tasks on CelebA, classifying gender and whether the subject is wearing glasses. To test representation editing, we remove the concept of ``eyewear'' or ``glasses'' from CelebA image representations by zeroing out any weight placed on these concepts in our \method decompositions and evaluate classifier performance. We report the performance of zero-shot classification and linear probes over our \method representation in Table \ref{tab:interv_glasses}. In both cases, we find that we can surgically target and remove information pertaining to glasses and reduce classifier performance while preserving information relevant to gender classification. We perform a similar experiment on the Waterbirds dataset \cite{sagawa2019distributionally} to remove spurious background signals in \ref{sec:app_waterbirds}. 
\section{Discussion} \label{sec:discussion}
In this work, we show that the information contained in CLIP embeddings can be approximated by a sparse, linear combination of simple semantic concepts, 
allowing us to interpret representations via sparse recovery. We propose \method\!, a method to transform the dense, uninterpretable embeddings of CLIP into human-interpretable sparse concept decompositions. 

We empirically demonstrate that \method allows for an adjustable tradeoff on the interpretability-accuracy Pareto frontier, enabling users to decide the loss in performance they are willing to incur for interpretability. Furthermore, we find that the improved interpretability of \method allows for users to diagnose and fix model mistakes, ideally increasing the effectiveness and performance of the overall system using a VLM. We then provide concrete use cases for \method\!: spurious correlation detection and model intervention and editing, showcasing the benefits of using interpretable embeddings with known semantic content. We highlight that \method embeddings can serve as post-hoc interpretations of CLIP embeddings and can even replace them to ensure full transparency. 

\paragraph{Broader Impact.} 
Similar to many works in the field of interpretability, our work provides greater understanding of the behavior of models, including but not limited to the broader implicit biases they perpetuate as well as mistakes made on individual samples. We believe this is particularly salient for CLIP, which is used in a variety of applications that are widely used in practice at this moment. We hope that insights gained from such interpretability allow users to make more informed decisions regarding how they interact with and use CLIP, regardless of their familiarity with machine learning or domain expertise in the task they are using CLIP for. We also highlight that \method can be used as a visualization-like tool for exploring and summarizing datasets at scale, allowing for easier auditing of spurious correlations and biases in both datasets and models. 

\paragraph{Limitations.} In this work, we use a large, overcomplete dictionary of one- and two-word concepts, however future work may wish to expand this dictionary or learn a dictionary over tokens (in discrete language space), to capture concepts with more than two words. This may also reduce the size of the dictionary and improve computation time. We note that this dictionary was constructed by looking at token frequency in the LAION-5B dataset, which has its own biases and may not correctly capture all the salient concepts that CLIP encodes. Despite this, we find that \method~performs well on a variety of tasks while outperforming state-of-the-art concept dictionaries (Fig.~\ref{fig:ablations}, Appendix Fig.~\ref{fig:vocab_abl}) and thus we believe LAION is a good dataset to generate a concept vocabulary from. We also note that this vocabulary can be easily modified by practitioners to consider additional concepts as needed for specific use cases.
Finally, \method also uses an $\ell_1$ penalty as the relaxation for $\ell_0$ regularization, but future work may consider alternative relaxations or even binary concept weights.

\newpage
\section*{Acknowledgements and Disclosure of Funding}

This work is supported in part by the NSF awards IIS-2008461, IIS-2040989, IIS-2238714, FAI-2040880, and research awards from Google, JP Morgan, Amazon, Adobe, Harvard Data Science Initiative, and the Digital, Data, and Design (D$^3$) Institute at Harvard. AO is supported by the National Science Foundation Graduate Research Fellowship under Grant No. DGE-2140743, and UB is funded by the Kempner Institute Graduate Research Fellowship. The views expressed here are those of the authors and do not reflect the official policy or position of the funding agencies.

\bibliographystyle{unsrtnat}

\bibliography{bibliography}

\begin{thebibliography}{69}
\providecommand{\natexlab}[1]{#1}
\providecommand{\url}[1]{\texttt{#1}}
\expandafter\ifx\csname urlstyle\endcsname\relax
  \providecommand{\doi}[1]{doi: #1}\else
  \providecommand{\doi}{doi: \begingroup \urlstyle{rm}\Url}\fi

\bibitem[Radford et~al.(2021)Radford, Kim, Hallacy, Ramesh, Goh, Agarwal, Sastry, Askell, Mishkin, Clark, Krueger, and Sutskever]{radford21alearning}
Alec Radford, Jong~Wook Kim, Chris Hallacy, Aditya Ramesh, Gabriel Goh, Sandhini Agarwal, Girish Sastry, Amanda Askell, Pamela Mishkin, Jack Clark, Gretchen Krueger, and Ilya Sutskever.
\newblock Learning transferable visual models from natural language supervision.
\newblock In Marina Meila and Tong Zhang, editors, \emph{Proceedings of the 38th International Conference on Machine Learning}, volume 139 of \emph{Proceedings of Machine Learning Research}, pages 8748--8763. PMLR, 18--24 Jul 2021.

\bibitem[Mokady et~al.(2021)Mokady, Hertz, and Bermano]{mokady2021clipcap}
Ron Mokady, Amir Hertz, and Amit~H Bermano.
\newblock Clipcap: Clip prefix for image captioning.
\newblock \emph{arXiv preprint arXiv:2111.09734}, 2021.

\bibitem[Ramesh et~al.(2022)Ramesh, Dhariwal, Nichol, Chu, and Chen]{ramesh2022hierarchical}
Aditya Ramesh, Prafulla Dhariwal, Alex Nichol, Casey Chu, and Mark Chen.
\newblock Hierarchical text-conditional image generation with clip latents.
\newblock \emph{arXiv preprint arXiv:2204.06125}, 1\penalty0 (2):\penalty0 3, 2022.

\bibitem[Koh and Liang(2017)]{koh2017understanding}
Pang~Wei Koh and Percy Liang.
\newblock Understanding black-box predictions via influence functions.
\newblock In \emph{International conference on machine learning}, pages 1885--1894. PMLR, 2017.

\bibitem[Bengio(2013)]{bengio2013deep}
Yoshua Bengio.
\newblock Deep learning of representations: Looking forward.
\newblock In \emph{International conference on statistical language and speech processing}, pages 1--37. Springer, 2013.

\bibitem[Olah et~al.(2017)Olah, Mordvintsev, and Schubert]{olah2017feature}
Chris Olah, Alexander Mordvintsev, and Ludwig Schubert.
\newblock Feature visualization.
\newblock \emph{Distill}, 2\penalty0 (11):\penalty0 e7, 2017.

\bibitem[Kim et~al.(2018)Kim, Wattenberg, Gilmer, Cai, Wexler, Viegas, et~al.]{kim2018interpretability}
Been Kim, Martin Wattenberg, Justin Gilmer, Carrie Cai, James Wexler, Fernanda Viegas, et~al.
\newblock Interpretability beyond feature attribution: Quantitative testing with concept activation vectors (tcav).
\newblock In \emph{International conference on machine learning}, pages 2668--2677. PMLR, 2018.

\bibitem[Hsu et~al.(2023)Hsu, Dorrell, Whittington, Wu, and Finn]{hsu2023disentanglement}
Kyle Hsu, Will Dorrell, James~CR Whittington, Jiajun Wu, and Chelsea Finn.
\newblock Disentanglement via latent quantization.
\newblock \emph{arXiv preprint arXiv:2305.18378}, 2023.

\bibitem[Geirhos et~al.(2023)Geirhos, Zimmermann, Bilodeau, Brendel, and Kim]{geirhos2023don}
Robert Geirhos, Roland~S Zimmermann, Blair Bilodeau, Wieland Brendel, and Been Kim.
\newblock Don't trust your eyes: on the (un) reliability of feature visualizations.
\newblock \emph{arXiv preprint arXiv:2306.04719}, 2023.

\bibitem[Mikolov et~al.(2013)Mikolov, Yih, and Zweig]{mikolov2013linguistic}
Tom{\'a}{\v{s}} Mikolov, Wen-tau Yih, and Geoffrey Zweig.
\newblock Linguistic regularities in continuous space word representations.
\newblock In \emph{Proceedings of the 2013 conference of the north american chapter of the association for computational linguistics: Human language technologies}, pages 746--751, 2013.

\bibitem[Park et~al.(2023)Park, Choe, and Veitch]{park2023linear}
Kiho Park, Yo~Joong Choe, and Victor Veitch.
\newblock The linear representation hypothesis and the geometry of large language models.
\newblock \emph{arXiv preprint arXiv:2311.03658}, 2023.

\bibitem[Arora et~al.(2018)Arora, Li, Liang, Ma, and Risteski]{arora2018linear}
Sanjeev Arora, Yuanzhi Li, Yingyu Liang, Tengyu Ma, and Andrej Risteski.
\newblock Linear algebraic structure of word senses, with applications to polysemy.
\newblock \emph{Transactions of the Association for Computational Linguistics}, 6:\penalty0 483--495, 2018.

\bibitem[Arora et~al.(2016)Arora, Li, Liang, Ma, and Risteski]{arora2016latent}
Sanjeev Arora, Yuanzhi Li, Yingyu Liang, Tengyu Ma, and Andrej Risteski.
\newblock A latent variable model approach to pmi-based word embeddings.
\newblock \emph{Transactions of the Association for Computational Linguistics}, 4:\penalty0 385--399, 2016.

\bibitem[Faruqui et~al.(2015)Faruqui, Tsvetkov, Yogatama, Dyer, and Smith]{faruqui2015sparse}
Manaal Faruqui, Yulia Tsvetkov, Dani Yogatama, Chris Dyer, and Noah Smith.
\newblock Sparse overcomplete word vector representations.
\newblock \emph{arXiv preprint arXiv:1506.02004}, 2015.

\bibitem[Yuksekgonul et~al.(2022{\natexlab{a}})Yuksekgonul, Bianchi, Kalluri, Jurafsky, and Zou]{yuksekgonul2022and}
Mert Yuksekgonul, Federico Bianchi, Pratyusha Kalluri, Dan Jurafsky, and James Zou.
\newblock When and why vision-language models behave like bags-of-words, and what to do about it?
\newblock In \emph{The Eleventh International Conference on Learning Representations}, 2022{\natexlab{a}}.

\bibitem[Merullo et~al.(2022)Merullo, Castricato, Eickhoff, and Pavlick]{merullo2022linearly}
Jack Merullo, Louis Castricato, Carsten Eickhoff, and Ellie Pavlick.
\newblock Linearly mapping from image to text space.
\newblock \emph{arXiv preprint arXiv:2209.15162}, 2022.

\bibitem[Seth et~al.(2023)Seth, Hemani, and Agarwal]{seth2023dear}
Ashish Seth, Mayur Hemani, and Chirag Agarwal.
\newblock Dear: Debiasing vision-language models with additive residuals.
\newblock In \emph{Proceedings of the IEEE/CVF Conference on Computer Vision and Pattern Recognition}, pages 6820--6829, 2023.

\bibitem[Koh et~al.(2020)Koh, Nguyen, Tang, Mussmann, Pierson, Kim, and Liang]{koh2020concept}
Pang~Wei Koh, Thao Nguyen, Yew~Siang Tang, Stephen Mussmann, Emma Pierson, Been Kim, and Percy Liang.
\newblock Concept bottleneck models.
\newblock In \emph{International conference on machine learning}, pages 5338--5348. PMLR, 2020.

\bibitem[Lampert et~al.(2009)Lampert, Nickisch, and Harmeling]{lampert2009learning}
Christoph~H Lampert, Hannes Nickisch, and Stefan Harmeling.
\newblock Learning to detect unseen object classes by between-class attribute transfer.
\newblock In \emph{2009 IEEE conference on computer vision and pattern recognition}, pages 951--958. IEEE, 2009.

\bibitem[Torresani et~al.(2010)Torresani, Szummer, and Fitzgibbon]{torresani2010efficient}
Lorenzo Torresani, Martin Szummer, and Andrew Fitzgibbon.
\newblock Efficient object category recognition using classemes.
\newblock In \emph{Computer Vision--ECCV 2010: 11th European Conference on Computer Vision, Heraklion, Crete, Greece, September 5-11, 2010, Proceedings, Part I 11}, pages 776--789. Springer, 2010.

\bibitem[Kumar et~al.(2009)Kumar, Berg, Belhumeur, and Nayar]{kumar2009attribute}
Neeraj Kumar, Alexander~C Berg, Peter~N Belhumeur, and Shree~K Nayar.
\newblock Attribute and simile classifiers for face verification.
\newblock In \emph{2009 IEEE 12th international conference on computer vision}, pages 365--372. IEEE, 2009.

\bibitem[Chattopadhyay et~al.(2023)Chattopadhyay, Pilgrim, and Vidal]{chattopadhyay2023information}
Aditya Chattopadhyay, Ryan Pilgrim, and Rene Vidal.
\newblock Information maximization perspective of orthogonal matching pursuit with applications to explainable ai.
\newblock In \emph{Thirty-seventh Conference on Neural Information Processing Systems}, 2023.

\bibitem[Oikarinen et~al.(2023)Oikarinen, Das, Nguyen, and Weng]{oikarinen2023label}
Tuomas Oikarinen, Subhro Das, Lam~M Nguyen, and Tsui-Wei Weng.
\newblock Label-free concept bottleneck models.
\newblock \emph{arXiv preprint arXiv:2304.06129}, 2023.

\bibitem[Panousis et~al.(2023)Panousis, Ienco, and Marcos]{panousis2023sparse}
Konstantinos~Panagiotis Panousis, Dino Ienco, and Diego Marcos.
\newblock Sparse linear concept discovery models.
\newblock In \emph{Proceedings of the IEEE/CVF International Conference on Computer Vision}, pages 2767--2771, 2023.

\bibitem[Olah et~al.(2020)Olah, Cammarata, Schubert, Goh, Petrov, and Carter]{olah2020zoom}
Chris Olah, Nick Cammarata, Ludwig Schubert, Gabriel Goh, Michael Petrov, and Shan Carter.
\newblock Zoom in: An introduction to circuits.
\newblock \emph{Distill}, 5\penalty0 (3):\penalty0 e00024--001, 2020.

\bibitem[Bau et~al.(2017)Bau, Zhou, Khosla, Oliva, and Torralba]{bau2017network}
David Bau, Bolei Zhou, Aditya Khosla, Aude Oliva, and Antonio Torralba.
\newblock Network dissection: Quantifying interpretability of deep visual representations.
\newblock In \emph{Proceedings of the IEEE conference on computer vision and pattern recognition}, pages 6541--6549, 2017.

\bibitem[Fong and Vedaldi(2018)]{fong2018net2vec}
Ruth Fong and Andrea Vedaldi.
\newblock Net2vec: Quantifying and explaining how concepts are encoded by filters in deep neural networks.
\newblock In \emph{Proceedings of the IEEE conference on computer vision and pattern recognition}, pages 8730--8738, 2018.

\bibitem[McGrath et~al.(2022)McGrath, Kapishnikov, Toma{\v{s}}ev, Pearce, Wattenberg, Hassabis, Kim, Paquet, and Kramnik]{mcgrath2022acquisition}
Thomas McGrath, Andrei Kapishnikov, Nenad Toma{\v{s}}ev, Adam Pearce, Martin Wattenberg, Demis Hassabis, Been Kim, Ulrich Paquet, and Vladimir Kramnik.
\newblock Acquisition of chess knowledge in alphazero.
\newblock \emph{Proceedings of the National Academy of Sciences}, 119\penalty0 (47):\penalty0 e2206625119, 2022.

\bibitem[Lucieri et~al.(2020)Lucieri, Bajwa, Braun, Malik, Dengel, and Ahmed]{lucieri2020interpretability}
Adriano Lucieri, Muhammad~Naseer Bajwa, Stephan~Alexander Braun, Muhammad~Imran Malik, Andreas Dengel, and Sheraz Ahmed.
\newblock On interpretability of deep learning based skin lesion classifiers using concept activation vectors.
\newblock In \emph{2020 international joint conference on neural networks (IJCNN)}, pages 1--10. IEEE, 2020.

\bibitem[Zhou et~al.(2018)Zhou, Sun, Bau, and Torralba]{zhou2018interpretable}
Bolei Zhou, Yiyou Sun, David Bau, and Antonio Torralba.
\newblock Interpretable basis decomposition for visual explanation.
\newblock In \emph{Proceedings of the European Conference on Computer Vision (ECCV)}, pages 119--134, 2018.

\bibitem[Fel et~al.(2023)Fel, Boutin, Moayeri, Cad{\`e}ne, Bethune, Chalvidal, Serre, et~al.]{fel2023holistic}
Thomas Fel, Victor Boutin, Mazda Moayeri, R{\'e}mi Cad{\`e}ne, Louis Bethune, Mathieu Chalvidal, Thomas Serre, et~al.
\newblock A holistic approach to unifying automatic concept extraction and concept importance estimation.
\newblock \emph{arXiv preprint arXiv:2306.07304}, 2023.

\bibitem[Bricken et~al.(2023)Bricken, Templeton, Batson, Chen, Jermyn, Conerly, Turner, Anil, Denison, Askell, et~al.]{bricken2023towards}
Trenton Bricken, Adly Templeton, Joshua Batson, Brian Chen, Adam Jermyn, Tom Conerly, Nick Turner, Cem Anil, Carson Denison, Amanda Askell, et~al.
\newblock Towards monosemanticity: Decomposing language models with dictionary learning.
\newblock \emph{Transformer Circuits Thread}, page~2, 2023.

\bibitem[Chen et~al.(2018)Chen, Li, Grosse, and Duvenaud]{chen2018isolating}
Ricky~TQ Chen, Xuechen Li, Roger~B Grosse, and David~K Duvenaud.
\newblock Isolating sources of disentanglement in variational autoencoders.
\newblock \emph{Advances in neural information processing systems}, 31, 2018.

\bibitem[Comon(1994)]{comon1994independent}
Pierre Comon.
\newblock Independent component analysis, a new concept?
\newblock \emph{Signal processing}, 36\penalty0 (3):\penalty0 287--314, 1994.

\bibitem[Hyv{\"a}rinen and Oja(2000)]{hyvarinen2000independent}
Aapo Hyv{\"a}rinen and Erkki Oja.
\newblock Independent component analysis: algorithms and applications.
\newblock \emph{Neural networks}, 13\penalty0 (4-5):\penalty0 411--430, 2000.

\bibitem[Makelov et~al.(2023)Makelov, Lange, and Nanda]{makelov2023subspace}
Aleksandar Makelov, Georg Lange, and Neel Nanda.
\newblock Is this the subspace you are looking for? an interpretability illusion for subspace activation patching.
\newblock \emph{arXiv preprint arXiv:2311.17030}, 2023.

\bibitem[Moayeri et~al.(2023)Moayeri, Rezaei, Sanjabi, and Feizi]{moayeri2023text}
Mazda Moayeri, Keivan Rezaei, Maziar Sanjabi, and Soheil Feizi.
\newblock Text-to-concept (and back) via cross-model alignment.
\newblock In \emph{International Conference on Machine Learning}, pages 25037--25060. PMLR, 2023.

\bibitem[Yuksekgonul et~al.(2022{\natexlab{b}})Yuksekgonul, Wang, and Zou]{yuksekgonul2022post}
Mert Yuksekgonul, Maggie Wang, and James Zou.
\newblock Post-hoc concept bottleneck models.
\newblock \emph{arXiv preprint arXiv:2205.15480}, 2022{\natexlab{b}}.

\bibitem[Yun et~al.(2022)Yun, Bhalla, Pavlick, and Sun]{yun2022vision}
Tian Yun, Usha Bhalla, Ellie Pavlick, and Chen Sun.
\newblock Do vision-language pretrained models learn composable primitive concepts?
\newblock \emph{arXiv preprint arXiv:2203.17271}, 2022.

\bibitem[Gandelsman et~al.(2023)Gandelsman, Efros, and Steinhardt]{gandelsman2023interpreting}
Yossi Gandelsman, Alexei~A Efros, and Jacob Steinhardt.
\newblock Interpreting clip's image representation via text-based decomposition.
\newblock \emph{arXiv preprint arXiv:2310.05916}, 2023.

\bibitem[Grootendorst(2022)]{grootendorst2022bertopic}
Maarten Grootendorst.
\newblock Bertopic: Neural topic modeling with a class-based tf-idf procedure.
\newblock \emph{arXiv preprint arXiv:2203.05794}, 2022.

\bibitem[Murphy et~al.(2012)Murphy, Talukdar, and Mitchell]{murphy2012learning}
Brian Murphy, Partha Talukdar, and Tom Mitchell.
\newblock Learning effective and interpretable semantic models using non-negative sparse embedding.
\newblock In \emph{Proceedings of COLING 2012}, pages 1933--1950, 2012.

\bibitem[Fyshe et~al.(2014)Fyshe, Talukdar, Murphy, and Mitchell]{fyshe2014interpretable}
Alona Fyshe, Partha~P Talukdar, Brian Murphy, and Tom~M Mitchell.
\newblock Interpretable semantic vectors from a joint model of brain-and text-based meaning.
\newblock In \emph{Proceedings of the conference. Association for Computational Linguistics. Meeting}, volume 2014, page 489. NIH Public Access, 2014.

\bibitem[Fyshe et~al.(2015)Fyshe, Wehbe, Talukdar, Murphy, and Mitchell]{fyshe2015compositional}
Alona Fyshe, Leila Wehbe, Partha Talukdar, Brian Murphy, and Tom Mitchell.
\newblock A compositional and interpretable semantic space.
\newblock In \emph{Proceedings of the 2015 conference of the north american chapter of the association for computational linguistics: Human language technologies}, pages 32--41, 2015.

\bibitem[Olshausen and Field(1997)]{olshausen1997sparse}
Bruno~A Olshausen and David~J Field.
\newblock Sparse coding with an overcomplete basis set: A strategy employed by v1?
\newblock \emph{Vision research}, 37\penalty0 (23):\penalty0 3311--3325, 1997.

\bibitem[Olshausen and Field(1996)]{olshausen1996emergence}
Bruno~A Olshausen and David~J Field.
\newblock Emergence of simple-cell receptive field properties by learning a sparse code for natural images.
\newblock \emph{Nature}, 381\penalty0 (6583):\penalty0 607--609, 1996.

\bibitem[Ramaswamy et~al.(2022)Ramaswamy, Kim, Fong, and Russakovsky]{ramaswamy2022overlooked}
Vikram~V Ramaswamy, Sunnie~SY Kim, Ruth Fong, and Olga Russakovsky.
\newblock Overlooked factors in concept-based explanations: Dataset choice, concept salience, and human capability.
\newblock \emph{arXiv preprint arXiv:2207.09615}, 2022.

\bibitem[Ribeiro et~al.(2016)Ribeiro, Singh, and Guestrin]{ribeiro2016should}
Marco~Tulio Ribeiro, Sameer Singh, and Carlos Guestrin.
\newblock " why should i trust you?" explaining the predictions of any classifier.
\newblock In \emph{Proceedings of the 22nd ACM SIGKDD international conference on knowledge discovery and data mining}, pages 1135--1144, 2016.

\bibitem[Vinson and Vigliocco(2008)]{vinson2008semantic}
David~P Vinson and Gabriella Vigliocco.
\newblock Semantic feature production norms for a large set of objects and events.
\newblock \emph{Behavior Research Methods}, 40\penalty0 (1):\penalty0 183--190, 2008.

\bibitem[McRae et~al.(2005)McRae, Cree, Seidenberg, and McNorgan]{mcrae2005semantic}
Ken McRae, George~S Cree, Mark~S Seidenberg, and Chris McNorgan.
\newblock Semantic feature production norms for a large set of living and nonliving things.
\newblock \emph{Behavior research methods}, 37\penalty0 (4):\penalty0 547--559, 2005.

\bibitem[Garrard et~al.(2001)Garrard, Lambon~Ralph, Hodges, and Patterson]{garrard2001prototypicality}
Peter Garrard, Matthew~A Lambon~Ralph, John~R Hodges, and Karalyn Patterson.
\newblock Prototypicality, distinctiveness, and intercorrelation: Analyses of the semantic attributes of living and nonliving concepts.
\newblock \emph{Cognitive neuropsychology}, 18\penalty0 (2):\penalty0 125--174, 2001.

\bibitem[Schuhmann et~al.(2021)Schuhmann, Vencu, Beaumont, Kaczmarczyk, Mullis, Katta, Coombes, Jitsev, and Komatsuzaki]{schuhmann2021laion}
Christoph Schuhmann, Richard Vencu, Romain Beaumont, Robert Kaczmarczyk, Clayton Mullis, Aarush Katta, Theo Coombes, Jenia Jitsev, and Aran Komatsuzaki.
\newblock Laion-400m: Open dataset of clip-filtered 400 million image-text pairs.
\newblock \emph{arXiv preprint arXiv:2111.02114}, 2021.

\bibitem[Liang et~al.(2022)Liang, Zhang, Kwon, Yeung, and Zou]{liang2022mind}
Victor~Weixin Liang, Yuhui Zhang, Yongchan Kwon, Serena Yeung, and James~Y Zou.
\newblock Mind the gap: Understanding the modality gap in multi-modal contrastive representation learning.
\newblock \emph{Advances in Neural Information Processing Systems}, 35:\penalty0 17612--17625, 2022.

\bibitem[Ilharco et~al.(2021)Ilharco, Wortsman, Wightman, Gordon, Carlini, Taori, Dave, Shankar, Namkoong, Miller, Hajishirzi, Farhadi, and Schmidt]{ilharco_gabriel_2021_5143773}
Gabriel Ilharco, Mitchell Wortsman, Ross Wightman, Cade Gordon, Nicholas Carlini, Rohan Taori, Achal Dave, Vaishaal Shankar, Hongseok Namkoong, John Miller, Hannaneh Hajishirzi, Ali Farhadi, and Ludwig Schmidt.
\newblock Openclip.
\newblock July 2021.
\newblock \doi{10.5281/zenodo.5143773}.
\newblock URL \url{https://doi.org/10.5281/zenodo.5143773}.
\newblock If you use this software, please cite it as below.

\bibitem[Krizhevsky et~al.(2009)Krizhevsky, Hinton, et~al.]{krizhevsky2009learning}
Alex Krizhevsky, Geoffrey Hinton, et~al.
\newblock Learning multiple layers of features from tiny images.
\newblock 2009.

\bibitem[Isola et~al.(2015)Isola, Lim, and Adelson]{StatesAndTransformations}
Phillip Isola, Joseph~J. Lim, and Edward~H. Adelson.
\newblock Discovering states and transformations in image collections.
\newblock In \emph{CVPR}, 2015.

\bibitem[Liu et~al.(2015)Liu, Luo, Wang, and Tang]{liu2015faceattributes}
Ziwei Liu, Ping Luo, Xiaogang Wang, and Xiaoou Tang.
\newblock Deep learning face attributes in the wild.
\newblock In \emph{Proceedings of International Conference on Computer Vision (ICCV)}, December 2015.

\bibitem[Lin et~al.(2014)Lin, Maire, Belongie, Hays, Perona, Ramanan, Doll{\'a}r, and Zitnick]{lin2014microsoft}
Tsung-Yi Lin, Michael Maire, Serge Belongie, James Hays, Pietro Perona, Deva Ramanan, Piotr Doll{\'a}r, and C~Lawrence Zitnick.
\newblock Microsoft coco: Common objects in context.
\newblock In \emph{Computer Vision--ECCV 2014: 13th European Conference, Zurich, Switzerland, September 6-12, 2014, Proceedings, Part V 13}, pages 740--755. Springer, 2014.

\bibitem[Deng et~al.(2009)Deng, Dong, Socher, Li, Li, and Fei-Fei]{imagenet_cvpr09}
J.~Deng, W.~Dong, R.~Socher, L.-J. Li, K.~Li, and L.~Fei-Fei.
\newblock {ImageNet: A Large-Scale Hierarchical Image Database}.
\newblock In \emph{CVPR09}, 2009.

\bibitem[Pedregosa et~al.(2011)Pedregosa, Varoquaux, Gramfort, Michel, Thirion, Grisel, Blondel, Prettenhofer, Weiss, Dubourg, Vanderplas, Passos, Cournapeau, Brucher, Perrot, and Duchesnay]{scikit-learn}
F.~Pedregosa, G.~Varoquaux, A.~Gramfort, V.~Michel, B.~Thirion, O.~Grisel, M.~Blondel, P.~Prettenhofer, R.~Weiss, V.~Dubourg, J.~Vanderplas, A.~Passos, D.~Cournapeau, M.~Brucher, M.~Perrot, and E.~Duchesnay.
\newblock Scikit-learn: Machine learning in {P}ython.
\newblock \emph{Journal of Machine Learning Research}, 12:\penalty0 2825--2830, 2011.

\bibitem[Beck and Teboulle(2009)]{beck2009fast}
Amir Beck and Marc Teboulle.
\newblock A fast iterative shrinkage-thresholding algorithm for linear inverse problems.
\newblock \emph{SIAM journal on imaging sciences}, 2\penalty0 (1):\penalty0 183--202, 2009.

\bibitem[Wu et~al.(2023)Wu, Yuksekgonul, Zhang, and Zou]{wu2023discover}
Shirley Wu, Mert Yuksekgonul, Linjun Zhang, and James Zou.
\newblock Discover and cure: Concept-aware mitigation of spurious correlation.
\newblock In \emph{International Conference on Machine Learning}, pages 37765--37786. PMLR, 2023.

\bibitem[Sagawa et~al.(2019)Sagawa, Koh, Hashimoto, and Liang]{sagawa2019distributionally}
Shiori Sagawa, Pang~Wei Koh, Tatsunori~B Hashimoto, and Percy Liang.
\newblock Distributionally robust neural networks for group shifts: On the importance of regularization for worst-case generalization.
\newblock \emph{arXiv preprint arXiv:1911.08731}, 2019.

\bibitem[Boyd et~al.(2011)Boyd, Parikh, Chu, Peleato, Eckstein, et~al.]{boyd2011distributed}
Stephen Boyd, Neal Parikh, Eric Chu, Borja Peleato, Jonathan Eckstein, et~al.
\newblock Distributed optimization and statistical learning via the alternating direction method of multipliers.
\newblock \emph{Foundations and Trends{\textregistered} in Machine learning}, 3\penalty0 (1):\penalty0 1--122, 2011.

\bibitem[Fei-Fei et~al.(2006)Fei-Fei, Fergus, and Perona]{fei2006one}
Li~Fei-Fei, Robert Fergus, and Pietro Perona.
\newblock One-shot learning of object categories.
\newblock \emph{IEEE transactions on pattern analysis and machine intelligence}, 28\penalty0 (4):\penalty0 594--611, 2006.

\bibitem[Xiao et~al.(2010)Xiao, Hays, Ehinger, Oliva, and Torralba]{xiao2010sun}
Jianxiong Xiao, James Hays, Krista~A Ehinger, Aude Oliva, and Antonio Torralba.
\newblock Sun database: Large-scale scene recognition from abbey to zoo.
\newblock In \emph{2010 IEEE computer society conference on computer vision and pattern recognition}, pages 3485--3492. IEEE, 2010.

\bibitem[Coates et~al.(2011)Coates, Ng, and Lee]{coates2011analysis}
Adam Coates, Andrew Ng, and Honglak Lee.
\newblock An analysis of single-layer networks in unsupervised feature learning.
\newblock In \emph{Proceedings of the fourteenth international conference on artificial intelligence and statistics}, pages 215--223. JMLR Workshop and Conference Proceedings, 2011.

\bibitem[Everingham et~al.()Everingham, Van~Gool, Williams, Winn, and Zisserman]{pascal-voc-2007}
M.~Everingham, L.~Van~Gool, C.~K.~I. Williams, J.~Winn, and A.~Zisserman.
\newblock The {PASCAL} {V}isual {O}bject {C}lasses {C}hallenge 2007 {(VOC2007)} {R}esults.
\newblock http://www.pascal-network.org/challenges/VOC/voc2007/workshop/index.html.

\bibitem[Stark et~al.(2011)Stark, Krause, Pepik, Meger, Little, Schiele, and Koller]{stark2011fine}
Michael Stark, Jonathan Krause, Bojan Pepik, David Meger, James~J Little, Bernt Schiele, and Daphne Koller.
\newblock Fine-grained categorization for 3d scene understanding.
\newblock \emph{International Journal of Robotics Research}, 30\penalty0 (13):\penalty0 1543--1552, 2011.

\end{thebibliography}

\newpage
\appendix

\begin{center}
    \LARGE{\textbf{Appendix}}
\end{center}\normalsize

\section*{Summary of Appendix Results}

\begin{itemize}
    \item \textbf{\ref{sec:appendix_method}}. Further Details on the Method
        \begin{itemize}
            \item \textbf{\ref{sec:proof}}. When do Sparse Decompositions Exist?
            \item \textbf{\ref{sec:optimization_discussion}}. Relationship between cosine similarity and MSE optimization
            \item \textbf{\ref{sec:admm_app}}. ADMM for batched on-device LASSO optimization
            \item \textbf{\ref{sec:app_modality_alignment}}. Effect of Modality Alignment
            \item \textbf{\ref{sec:app_hardware}}. Experimental Details
        \end{itemize}

    \item \textbf{\ref{sec:add_results}}. Additional Results
            \begin{itemize}
                \item \textbf{\ref{sec:user_study}}. User Study for Human Interpretability
                \item \textbf{\ref{sec:app_probing}}. Performance of \method on Probing Tasks
                \item \textbf{\ref{sec:app_retrieval}}. Performance of \method on Retrieval Tasks
                \item \textbf{\ref{sec:add_zs}}. Additional Zero-Shot Results
                \item \textbf{\ref{sec:add_histograms}}. Additional ImageNet Concept Histograms
                \item \textbf{\ref{sec:app_spurious_correlations}}. Additional Case Study: Detecting Spurious Correlations
                \item \textbf{\ref{sec:app_wb_interv}}. Additional Case Study: Spurious Correlation Intervention
                \item \textbf{\ref{sec:app_dist_shift1}}. Additional Case Study: Distribution Shift Monitoring
                \item \textbf{\ref{sec:app_dist_shift2}}. Additional Case Study: Distribution Shift Monitoring
                \item \textbf{\ref{sec:app_negweights}}. Checking the Interpretability of Negative Concepts
                \item \textbf{\ref{sec:app_mean}}. Understanding the Image Mean for Modality Alignment
                \item \textbf{\ref{sec:app_vocab_choice}}. Choice of Concept Vocabulary
                \item \textbf{\ref{sec:app_concept_type_dist}}. Concept Type Distribution
                \item \textbf{\ref{sec:app_clip_rn50}}. Experiments on Alternative CLIP Architecture 
            \end{itemize}

\end{itemize}

\section{Further Details on the Method}
\label{sec:appendix_method}
\subsection{When do Sparse Decompositions Exist?}\label{sec:proof}

\begin{theorem} \label{thm:prop1}
    Given Assumptions 1-5, CLIP image embeddings $f$ can be written as a sparse linear combination of text embeddings, i.e,  
    \begin{align*}
        f(\X^\text{img}) = \mathbf{C}^\text{txt} \W; ~~s.t.~~\| \W \|_0 \leq \alpha
    \end{align*}
    where $\W \in \R_+^k$, and $\mathbf{C}^\text{txt} \in \R^{d \times k}$, which is the text concept dictionary defined previously.
\end{theorem}

\begin{proof} 
    Any vector $\omega$ can be written as $\omega = \sum_{i=1}^k \omega_i \E_i$, where $\omega_i \in \R_+$, and $\E_i \in \R^k$ is a one-hot vector with one at the $i^\text{th}$ co-ordinate. Thus we have 
    \begin{align*}
        f(\X^\text{img}) = &f \circ h^\text{img}(\omega, \epsilon) = f \circ h^\text{img}(\omega)~~~~(\text{Assumption~2})\\
        = &f \circ h^\text{img}\left( \sum_{i=1}^k \omega_i \E_i \right) = \sum_{i=1}^k \omega_i \underbrace{f \circ h^\text{img}(\E_i)}_{\mathbf{c}^\text{img}_i}~~~~(\text{Assumption~3})
    \end{align*}

    Here we define $\mathbf{c}_i^\text{img} = f \circ h^\text{img}(\E_i)$ as the `image' concept basis vector; analogous to the text concept basis vector $\mathbf{c}_i^\text{txt} = g \circ h^\text{txt}(\E_i)$ already defined. Thus Assumption 2 implies the existence of a sparse decomposition of $f$ in terms of `image' concept vectors $\mathbf{c}^\text{img}_i$. Additionally, \underline{Assumption 1} ensures that this decomposition is sparse, as $\omega$ is sparse. So far, we have $f(\X^\text{img}) = \mathbf{C}^\text{img} \omega~~~s.t.~~ \| \omega \|_0 \leq \alpha$.

    From \underline{Assumption 4}, the image concept vectors and text concept vectors are equal to each other, i.e, $\mathbf{c}_i^\text{img} = f \circ h^\text{img}(\E_i) = g \circ h^\text{txt}(\E_i) = \mathbf{c}_i^\text{txt}$. Finally, from \underline{Assumption 5}, we have that the text concept vectors $\mathbf{c}_i^\text{txt}$ are given simply by word embeddings $g$ of individual words.

    Stringing these arguments together, we have that image representations $f(\X^\text{img})$ can be written as a sparse linear combination of vectors obtain from CLIP word embeddings $\mathbf{c}_i^\text{txt}$. We finally set $\W = \omega$, thus proving the assertion.

\end{proof}

\subsection{Relationship between cosine similarity and MSE optimization.} \label{sec:optimization_discussion}

Recall our $\ell_1$ relaxed cosine similarity optimization problem from Eqn.~\eqref{eq:full_optimization},
\begin{align}
    \min_{\W \in \mathbb{R}_+^c} \|\W\|_0 
    ~~\text{s.t.}~~ \langle \Z, \frac{\C\W}{\|\C\W\|_2} \rangle \geq 1-\epsilon. 
\end{align}
First we relax the $\ell_0$ constraint to an $\ell_1$ penalty.
\begin{align}
    \max_{\W \in \mathbb{R}_+^c} \langle \Z,  \frac{\C\W}{\|\C\W\|_2} \rangle - \lambda\|\W\|_1.
\end{align}
By observing that $||x-y||_2^2 = \langle x-y, x-y \rangle = \langle x,x \rangle + \langle y, y \rangle - 2 \langle x,y \rangle$ and that $\Z$, $ \frac{\C\W}{\|\C\W\|_2}$ are unit-norm, maximizing the above inner product is equivalent to minimizing the euclidean norm,
\begin{align}
    \min_{\W \in \mathbb{R}_+^c} || \frac{\C\W}{\|\C\W\|_2} - \Z||_2^2 + 2\lambda||\W||_1. 
\end{align}
This is a non-convex problem, but we can relax this problem to achieve better reconstruction in terms of euclidan distance as shown in Eqn.~\eqref{eq:optimization},
\begin{align}
    \min_{\W \in \mathbb{R}_+^c} \|\C\W - \Z\|_2^2 + 2 \lambda \|\W\|_1. \label{eq:appendix_optimization}
\end{align}
This problem will optimize euclidean distance between $\C\W$ and $\Z$. Consider two vectors $x,y$ on the unit sphere such that $\langle \frac{x}{||x||}, \frac{y}{||y||} \rangle > 0$. While any vector $\alpha y$, $\alpha > 0$ will have the same cosine similarity score, the optimal vector in terms of euclidean distance to $x$ is the vector $\alpha y$ such that $\alpha = \text{proj}_y(x)$, or in other words the projection of $x$ onto $y$. Thus, solving for euclidean distance to approximate $x$ will find $\alpha y$ which we must then normalize to find the unit-norm solution $y$. This explains the normalizing process described in Section \ref{sec:modality_alignment}.

Additionally, we can view Eqn.~\eqref{eq:appendix_optimization} as applying shrinkage to $\C\W$. Reconverting from euclidean norm to inner product, Eqn.~\eqref{eq:appendix_optimization} becomes
\begin{align}
    \max_{\W \in \mathbb{R}_+^c} \langle \C\W, \Z\rangle - \frac{1}{2}\langle \C\W, \C\W \rangle - \lambda||\W||_1 = \langle \C\W, \Z\rangle - \frac{1}{2}||\C\W||^2_2 - \lambda||\W||_1.
\end{align}
In conclusion, our optimization problem maximizes the inner product while imposing a shinkage penalty and sparsity penalty. Empirically, our reconstructions $\C\W$ are low-norm, so we normalize after solving to recover the unit-norm reconstruction.

\subsection{ADMM for batched on-device LASSO optimization.}
\label{sec:admm_app}
As each decomposition requires solving a LASSO optimization problem, we implement the Alternating Direction Method of Multipliers (ADMM) algorithm in Pytorch over batches with GPU support for efficient decomposition of large scale datasets over large numbers of concepts \citep{boyd2011distributed}. In practice, ADMM achieves primal and dual tolerances of $1\mathrm{e}-4$ in fewer than $1000$ iterations on a batch size of $1024.$ We present an empirical comparison beterrn LASSO and ADMM in \ref{fig:admm_comparison}, where we find both methods to be approximately equivalent. 

Next we derive the iterates for our ADMM algorithm. Recall our optimization problem, 
\begin{align}
    \min_{\W \in \mathbb{R}_+^c} \|\C\W - \Z\|_2^2 + 2 \lambda \|\W\|_1. \label{eq:admm_optimization}
\end{align}
ADMM breaks down convex optimization problems into multiple sub-problems while penalizing the difference in solutions. We break Eqn.~\eqref{eq:admm_optimization} into two subproblems, one solving the euclidean distance objective and one solving the $\ell_1$ and nonnegativity constraint. We let $w$ denote the former solution, $z$ the latter, and $u$ tracks the difference between the two. Our ADMM iterates $(w^k, z^k, u^k)$ are
\begin{align}
    w^{k+1} &= \arg\min_w (f(w) + \frac{\rho}{2}||w^k-z^k+u^k||_2^2), \\
    z^{k+1} &= (S_{\lambda/\rho}(w^{k+1}+u^k))_+, \\
    u^{k+1} &= u^k + w^{k+1} - z^{k+1},
\end{align}
where $S_\kappa$ is a soft-thresholding function used to satisfy the LASSO constraints,
\begin{align}
    S_{\kappa}(a) := \begin{cases}
        a - \kappa, & a > \kappa \\
        0, & |a| \leq \kappa \\
        a + \kappa, & a < -\kappa
    \end{cases}
\end{align}
As our optimization function $f(w)$ is quadratic, we can analytically compute $w^{k+1}$ as
\begin{align}
    w^{k+1} = (2\mathbf{C}^T\mathbf{C} + \rho)^{-1}(\rho v + 2\mathbf{C}w),
\end{align}
where $v = z^k-u^k$. In our experiments we set $\rho=5$, and stop when tolerances $\epsilon_{\mathrm{prim}} = ||x^{k+1}-z^{k+1}||_2$, $\epsilon_{\mathrm{dual}} = ||\rho(z^{k+1}-z^k)||_2$ are less than $1\mathrm{e}-4$. Over a batch, we iterate until every solver in the batch has reached the above tolerances.

\begin{figure*}[h!]
\centering
\begin{subfigure}{0.33\linewidth}
   \includegraphics[width=\linewidth]{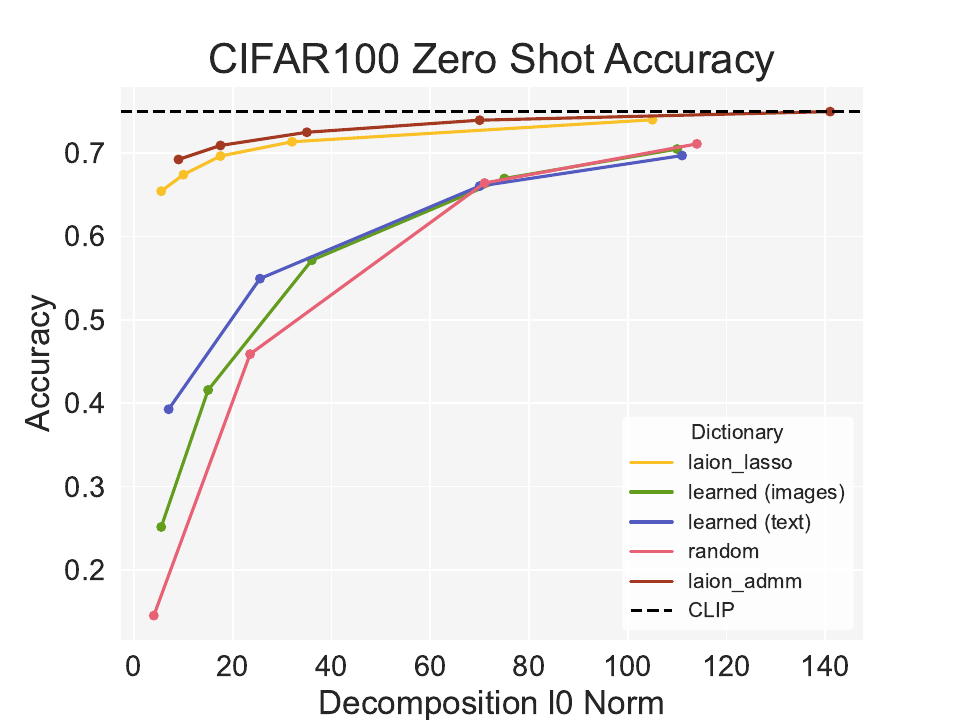}
\end{subfigure}
\begin{subfigure}{0.33\linewidth}
   \includegraphics[width=\linewidth]{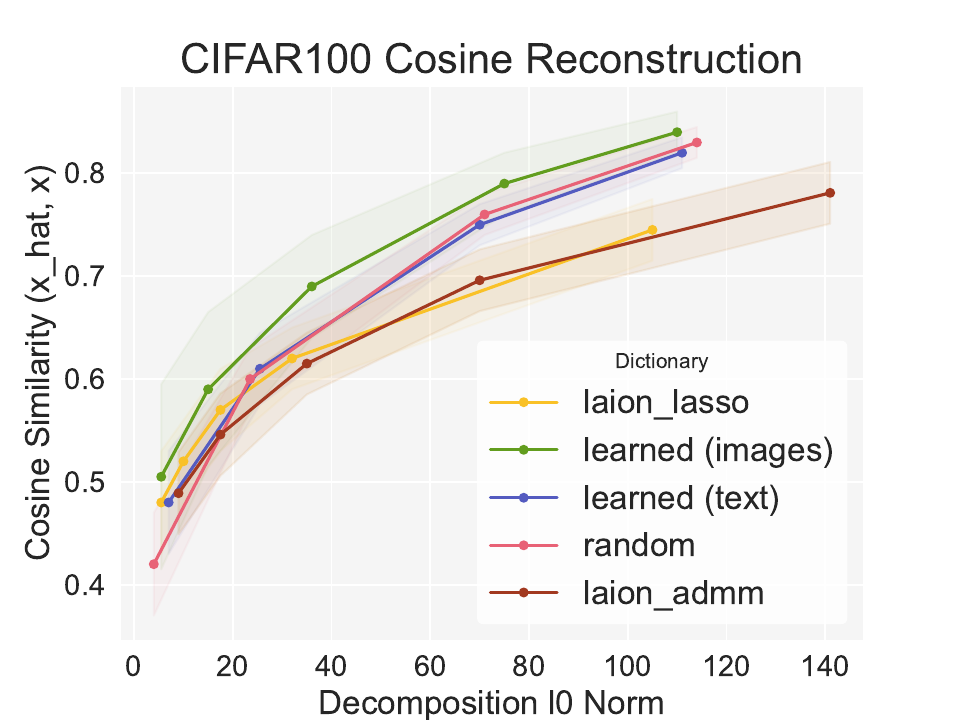}
\end{subfigure}
\caption{Comparison of ADMM (maroon) and LASSO (yellow) for solving the SpLiCE objective on zero shot accuracy (left) and cosine reconstruction (right) on CIFAR100. Both methods are approximately equal. 
}
\label{fig:admm_comparison}
\end{figure*}

\subsection{Effect of Modality Alignment}
\label{sec:app_modality_alignment}
We take MSCOCO images and captions, embed them with CLIP, and compare the cosine similarity between modalities and inter-modality. Before mean-centering and renormalizing, the similarity within modalities is high, with an average of around 0.3. This indicates that the image and text embeddings do not span the entire unit-sphere but rather lie on two cones. However, the similarity across modalities has an average concentrating around zero, indicating that these two cones are non-overlapping. However, after mean-centering and normalizing, we observe that the average cosine similarity for images, text, and between images and text becomes zero and the modalities are aligned. 
\begin{figure}[ht]
   \centering
   \begin{subfigure}[b]{0.35\columnwidth}
       \includegraphics[width=\columnwidth]{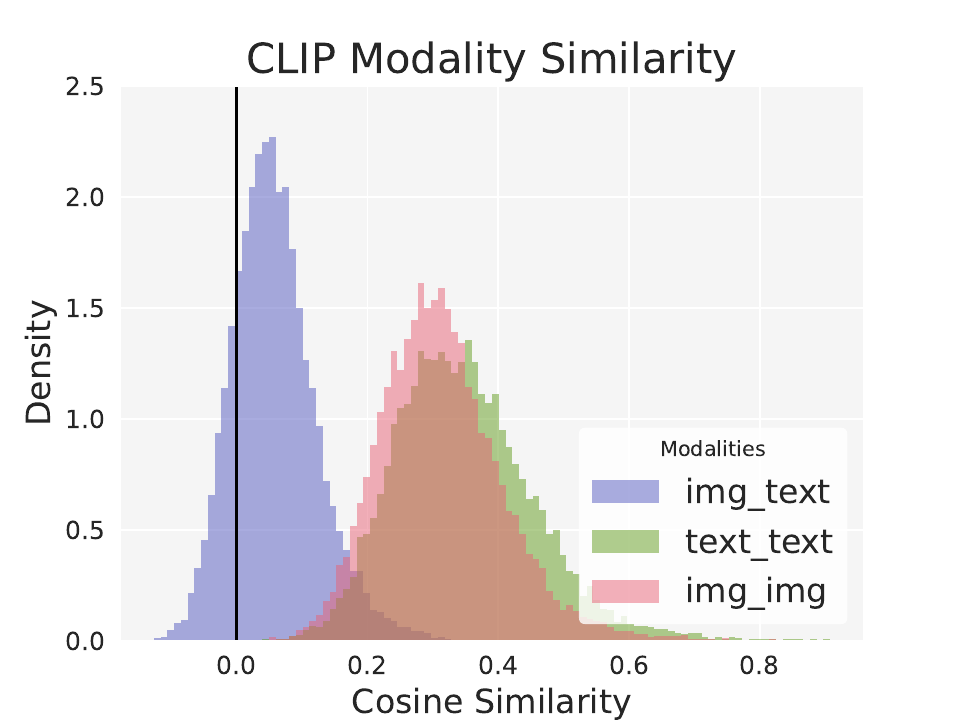}
   \end{subfigure}%
   \begin{subfigure}[b]{0.35\columnwidth}
       \includegraphics[width=\columnwidth]{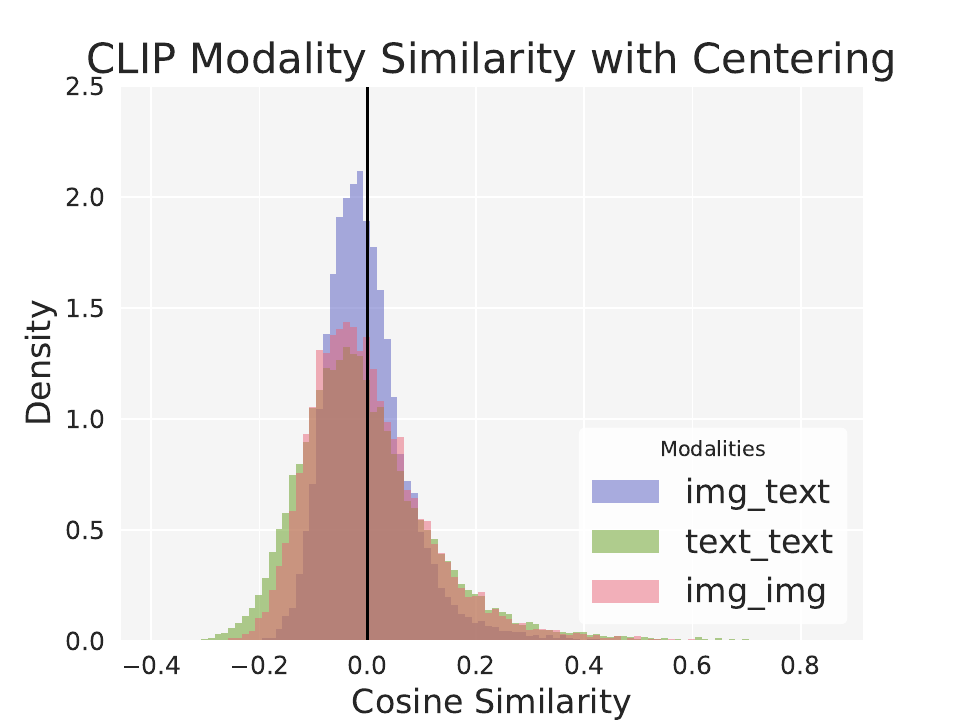}
   \end{subfigure}
   \caption{Average cosine similarity across pairs of image-text, image-image, and text-text data from MSCOCO. After aligning modalities, the distribution of similarities is centered around zero.}
   \label{fig:modality_alignment}
\end{figure}

\subsection{Experimental Details}
\label{sec:app_hardware}

All experiments are able to be performed on a single A100 GPU to run fast inference with CLIP. After embedding the concept dictionary, all computation can be performed on a CPU. Code is made available at \url{https://github.com/AI4LIFE-GROUP/SpLiCE}.  

\section{Additional Results}
\label{sec:add_results}

\subsection{User Study for Human Interpretability}
\label{sec:user_study}

We present results from a user study in \ref{fig:user_study} to assess the human interpretability of SpLiCE. We base our study off of that performed by \citep{oikarinen2023label} to evaluate Label-Free Concept Bottleneck Models (LF-CBMs). We benchmark our method against LF-CBMs and IP-OMP \citep{chattopadhyay2023information}. We provided users with twenty randomly chosen, correctly predicted images from ImageNet and explanations from two different methods comprising the top six most important concepts for every image. We then asked users to evaluate and compare the different concept-based explanations for (1) their relevance to the provided image inputs, (2) their relevance to model predictions, and (3) their informativeness on Likert scales from 1 to 5. We found that users significantly preferred explanations generated by SpLiCE to the two baselines for relevance to the images and informativeness, with significance determined via a one-sample two-sided t-test and a threshold of p=0.01. We also highlight that our method is able to produce similar/better concept decompositions, in terms of human interpretability, than the baselines without needing to train a classification probe or use class labels for concept mining, both of which are computationally expensive. This user study was ruled exempt by our institution's IRB, as no risks were posed to the users. Participants were able to opt out at any time, and no questions were asked regarding the participants themselves. 

\begin{figure*}[h!]
\centering
\begin{subfigure}{0.44\linewidth}
   \includegraphics[width=\linewidth]{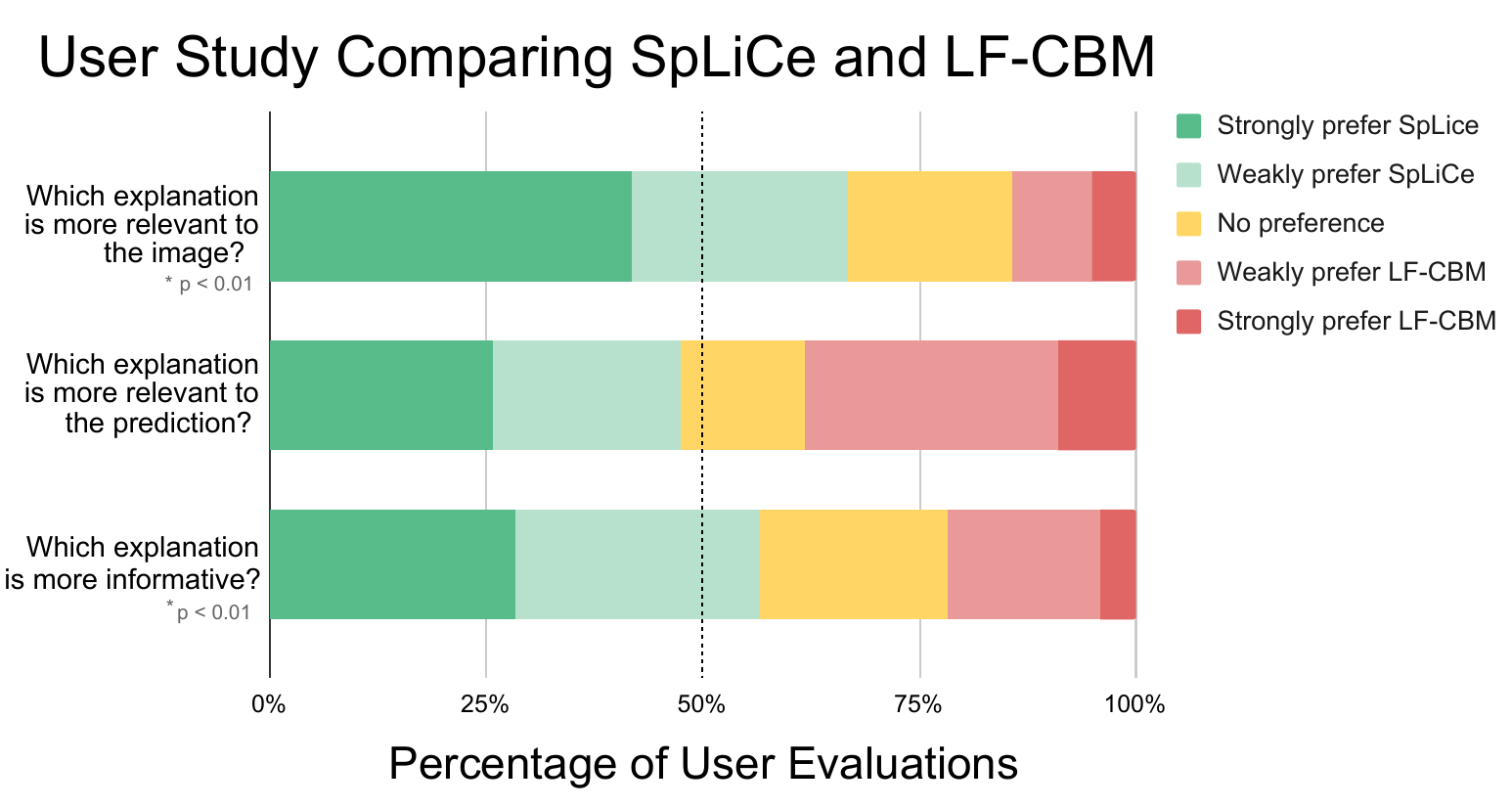}
\end{subfigure}
\begin{subfigure}{0.44\linewidth}
   \includegraphics[width=\linewidth]{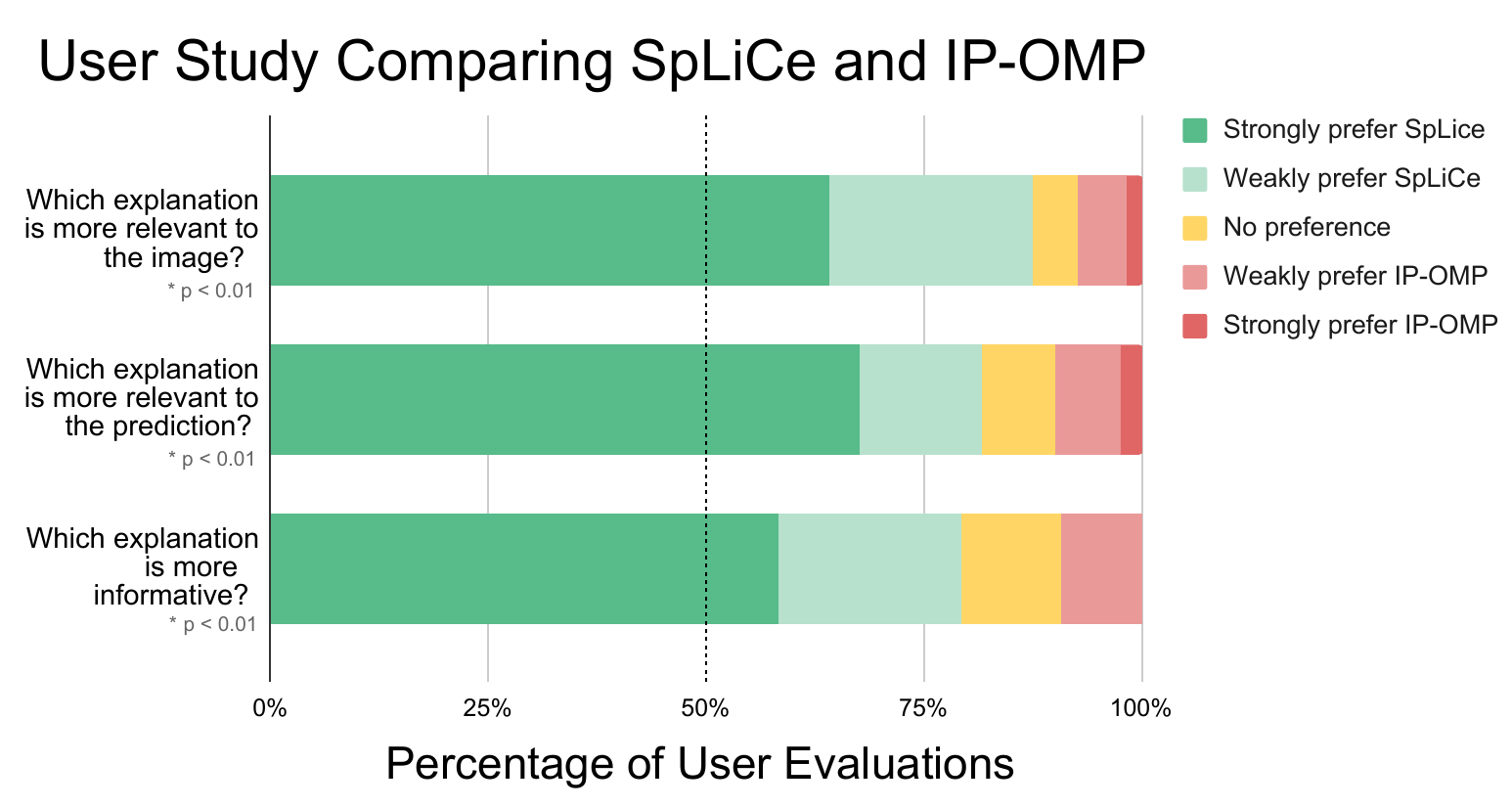}
\end{subfigure}
\caption{
Results of a user study evaluating SpLiCE, LF-CBM, and IP-OMP in the style of the user study from LF-CBM. Overall, we find that explanations generated by SpLiCE are deemed more relevant to the image, relevant to the prediction, and more informative than prior methods. 
}
\label{fig:user_study}
\end{figure*}

\subsection{Performance of \method on Probing Tasks}
\label{sec:app_probing}
We evaluate the performance
of the decompositions on probes trained on both regular
CLIP embeddings as well as decomposed CLIP embeddings for CIFAR100 in \ref{tab:decomp_probe_cifar} and MIT States in \ref{tab:decomp_probe_mit}. We consider two scenarios: a probe trained on CLIP embeddings and tested on \method embeddings of various sparsities (shown in row \texttt{CLIP Probe}), and a probe both trained and evaluated on \method embeddings (shown in row \texttt{\method Probe}). We report mean over three runs, with standard deviations for each experiment being less than 0.005. 
We find that \method representations closely match the performance of dense CLIP embeddings, with a slight drop in performance when probes are trained directly on \method embeddings rather than trained on CLIP embeddings and evaluated on \method embeddings for CIFAR100.

\begin{table}[!h]
\caption{Evaluation of Probing Performance on CIFAR100}
\label{tab:decomp_probe_cifar}
\begin{center}
\begin{small}
\begin{sc}
\begin{tabular}{lcccll}
\toprule
 &
  $l_0$ = 3 &
  $l_0$ = 6 &
  $l_0$ = 23 &
  $l_0$ = 117 &
  CLIP \\ \hline
\method Probe &
  0.95  &
  0.95  &
  0.95  &
  0.95  &
  -- \\
CLIP Probe &
  0.96  &
  0.96  &
  0.97  &
  0.97  &
  0.97  \\
\bottomrule
\end{tabular}
\end{sc}
\end{small}
\end{center}
\end{table}

\begin{table}[!h]
\caption{Evaluation of Probing Performance on MIT States}
\label{tab:decomp_probe_mit}
\begin{center}
\begin{small}
\begin{sc}
\begin{tabular}{lcccl}
\toprule
 &
  $l_0$ = 4 &
  $l_0$ = 7 &
  $l_0$ = 27 &
  CLIP \\ \hline
\method Probe &
  \begin{tabular}[c]{@{}c@{}}0.883       \end{tabular} &
  \begin{tabular}[c]{@{}c@{}}0.883    \end{tabular} &
  \begin{tabular}[c]{@{}l@{}}0.882      \end{tabular} &
  -- \\
CLIP Probe &
  \begin{tabular}[c]{@{}c@{}}0.883     \end{tabular} &
  \begin{tabular}[c]{@{}c@{}}0.883     \end{tabular} &
  \begin{tabular}[c]{@{}c@{}}0.884  \end{tabular} &
  \begin{tabular}[c]{@{}l@{}}0.883      \end{tabular} \\
\bottomrule
\end{tabular}
\end{sc}
\end{small}
\end{center}
\end{table}

\subsection{\method Performance on Retrieval Tasks}
\label{sec:app_retrieval}
We test the performance of \method embeddings on text-to-image and image-to-text retrieval tasks. We evaluate retrieval over various 1024 sample subsets of MSCOCO, and assess recall performance for the top-k closest embeddings of the opposite modality for $k=\{1, 5, 10\}$. 
We find that our semantic concept dictionaries outperform all baselines when decomposition sparsity is high, but that dictionaries learned over images perform slightly better for text to image retrieval when decompositions have greater than 30 nonzero concepts. 

\begin{figure}[!h]
    \centering
    \begin{subfigure}{0.32\columnwidth}
        \includegraphics[width=\columnwidth]{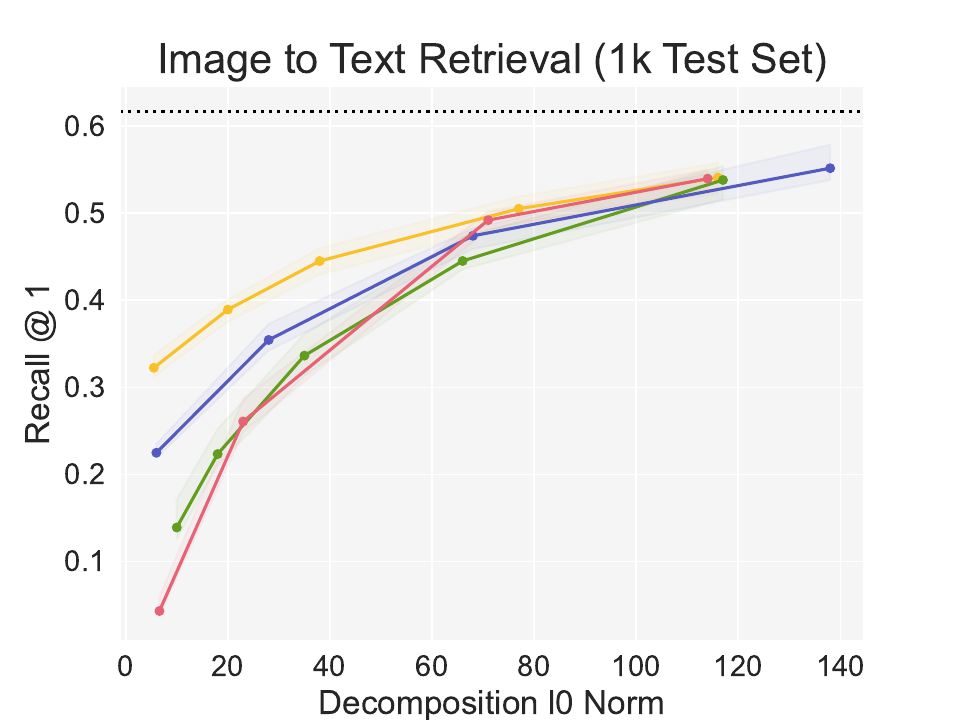}
    \end{subfigure}
    \begin{subfigure}{0.32\columnwidth}
        \includegraphics[width=\columnwidth]{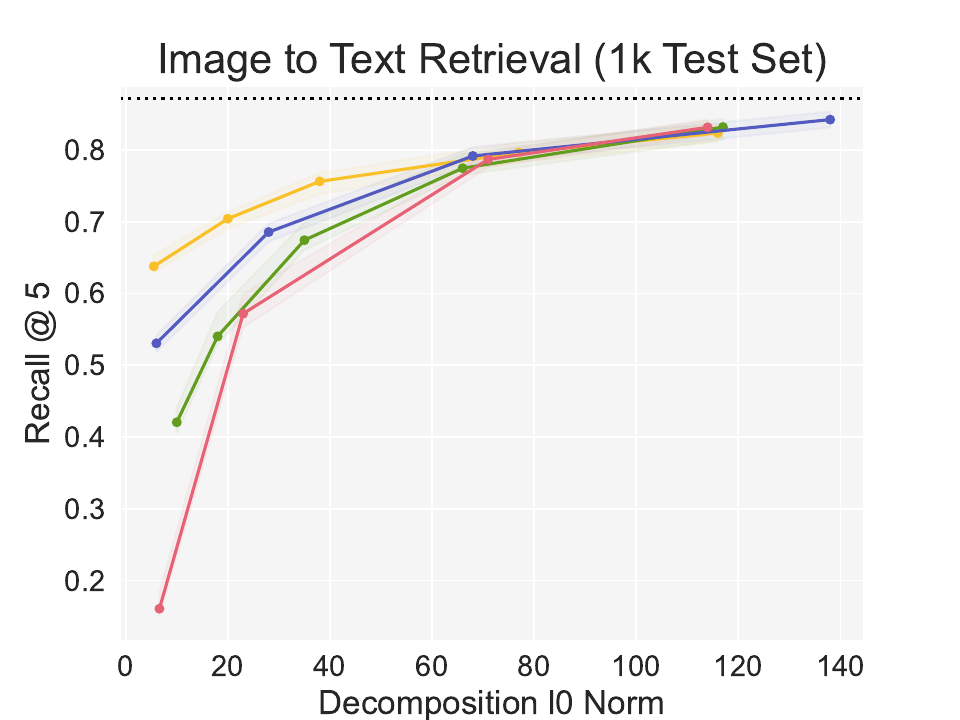}
    \end{subfigure}
    \begin{subfigure}{0.32\columnwidth}
        \includegraphics[width=\columnwidth]{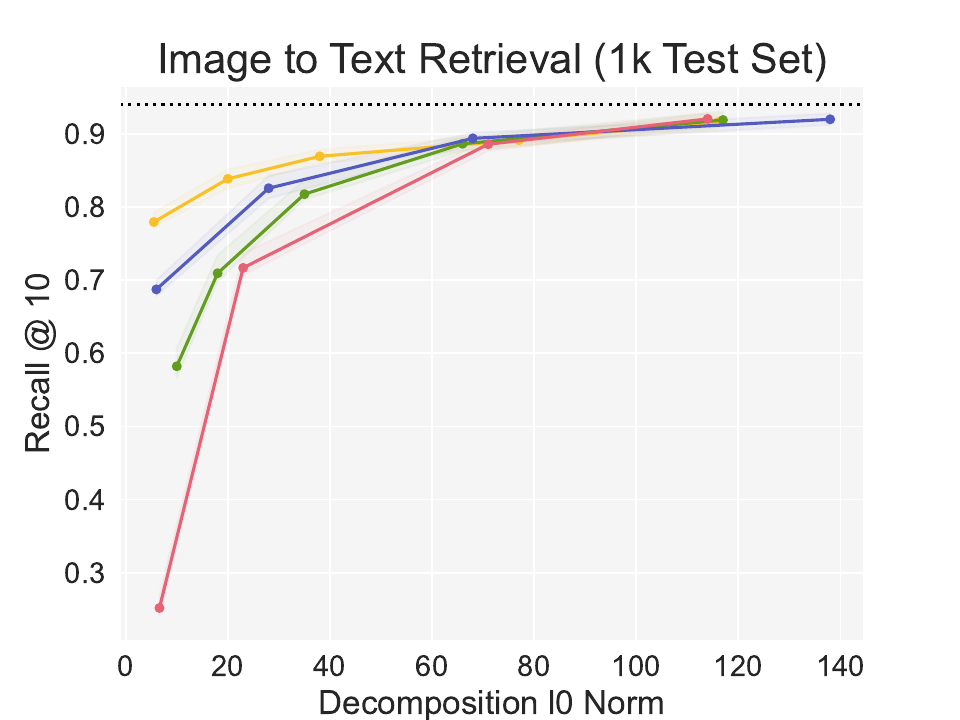}
    \end{subfigure}
    \begin{subfigure}{0.32\columnwidth}
        \includegraphics[width=\columnwidth]{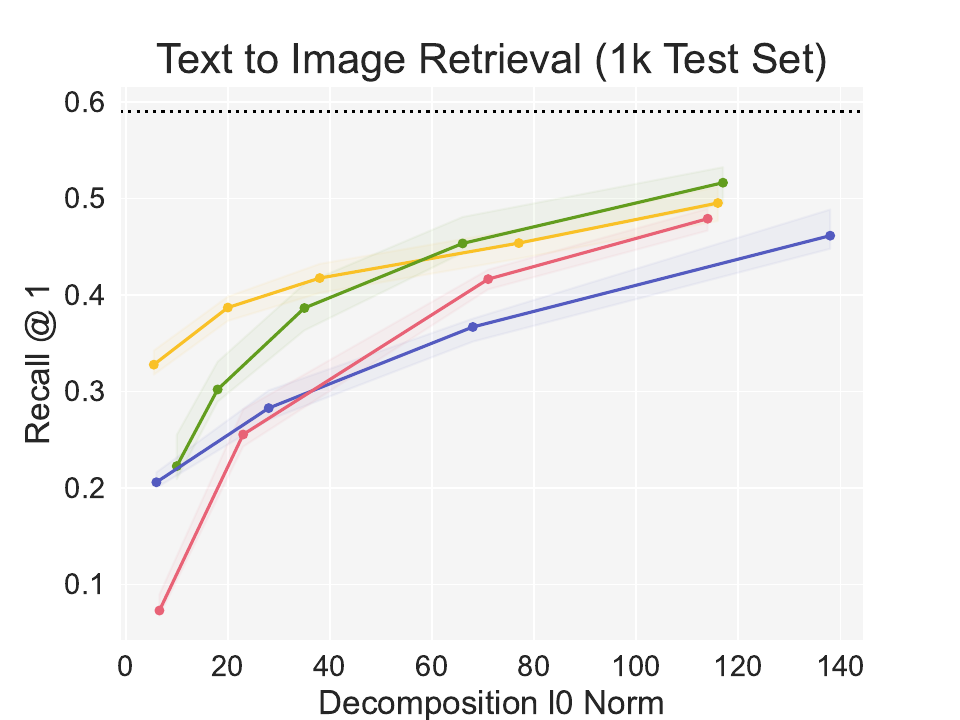}
    \end{subfigure}
    \begin{subfigure}{0.32\columnwidth}
        \includegraphics[width=\columnwidth]{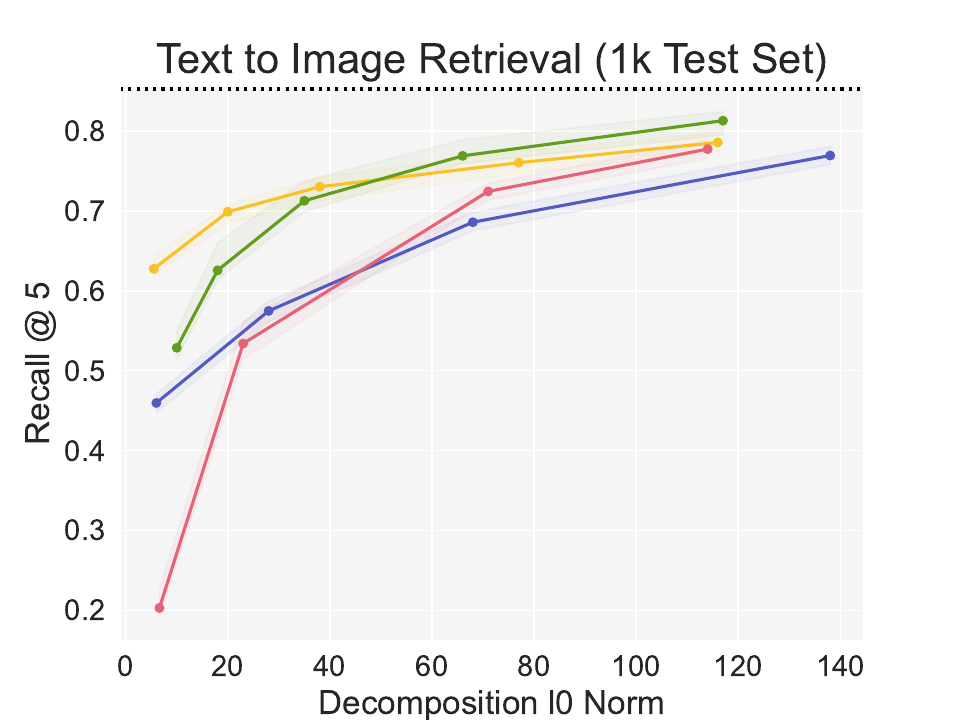}
    \end{subfigure}
    \begin{subfigure}{0.32\columnwidth}
        \includegraphics[width=\columnwidth]{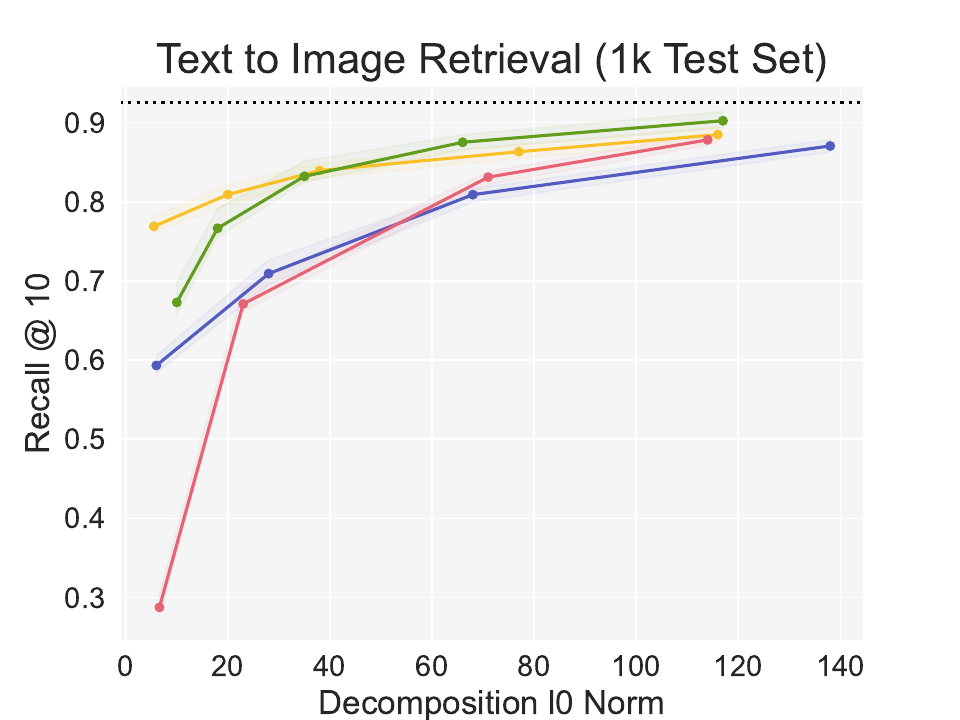}
    \end{subfigure}
    \begin{subfigure}{0.9\linewidth}
   \includegraphics[width=\linewidth]{figures/legend2.pdf}
    \end{subfigure}
    \caption{Top-1 , 5, 10 performance of SpLiCE representations on image-to-text (top) and text-to-image (bottom) retrieval on MSCOCO.}
    \label{fig:retrieval_extra}
\end{figure}

\subsection{Additional Zero-Shot Results}
\label{sec:add_zs}
We present additional results comparing SpLiCE reconstructed vectors and CLIP embeddings on the Caltech101 \citep{fei2006one}, SUN397 \citep{xiao2010sun}, STL10 \citep{coates2011analysis}, and VOC2007 \citep{pascal-voc-2007} datasets in \ref{tab:additional_datasets}. We use SpLiCE decompositions with sparsities of 20-35, and we find that they are comparable to the unaltered CLIP embeddings.

\begin{table}[h!]
        \centering
        \caption{Additional zero-shot accuracy on baselines from the CLIP paper, for decompositions of sparsity 20-35. Note that at human-interpretable levels of sparsity, we see a minor drop in performance.}
        \label{tab:additional_datasets}
        \begin{small}
        \begin{sc}
        \begin{tabular}{lccccr}
            \toprule
                & Caltech101 & SUN397 & STL10 & VOC 2007 \\ 
            \midrule
            CLIP Reported  &  0.88  &  0.63 & 0.97 & 0.83  \\
            CLIP Implemented &  0.90  &  0.67 & 0.96 & 0.92  \\
            SpLiCE &  0.86  & 0.66 & 0.96 & 0.83 \\
            \bottomrule
        \end{tabular}
        \end{sc}
       \end{small}
\end{table}

We further explore the performance of SpLiCE decompositions in the limit as they approach the sparsity of the baseline CLIP embeddings (512). We find that SpLiCE completely recovers CLIP zero-shot accuracy at this limit, as shown in \ref{tab:zero_shot_limit}.

\begin{table}[h!]
        \centering
        \begin{small}
        \begin{sc}
        \caption{Zero shot performance at sparsity 512. Note that SpLiCE completely recovers baseline CLIP zero shot accuracy.}
        \label{tab:zero_shot_limit}
        \begin{tabular}{lccc}
            \toprule
                & CIFAR100 & MITStates & Imagenet  \\ 
            \midrule
            CLIP Baseline &  0.750  &  0.469 & 0.552 \\
            SpLiCE (512) &  0.768  & 0.474 & 0.552 \\
            \bottomrule
        \end{tabular}
        \end{sc}
        \end{small}
\end{table}

\subsection{Additional ImageNet Concept Histograms}
\label{sec:add_histograms}
We present concept histograms for the top seven concepts of five more ImageNet classes: \{`Face Powder', `Feather Boa', `Jack-O'-Lantern', `Kimono', `Dalmation'\}, similar to Figure \ref{fig:extra_clip_decomps}. These decompositions give insights both into the distribution of each class as well as some biases of CLIP. For example, for the class `Face Powder', the concept ``benefit" is the fifth most common concept, and it is indeed a common cosmetic brand name in the images. For the `Dalmation' class, we see that the decompositions consists of concepts relating to dogs and black and white spots, which together make up the high-level concept of a dalmation. Finally, for the class `Kimono', the concept ``doll" is the seventh most common, although all of the images in the `Kimono' class were of real humans, not of dolls. This highlights an implicit bias in CLIP's representations or in the descriptions of people wearing kimonos in CLIP's training set.

\begin{figure}[!ht]
    \centering
        \begin{subfigure}{0.32\columnwidth}
        \includegraphics[width=\columnwidth]{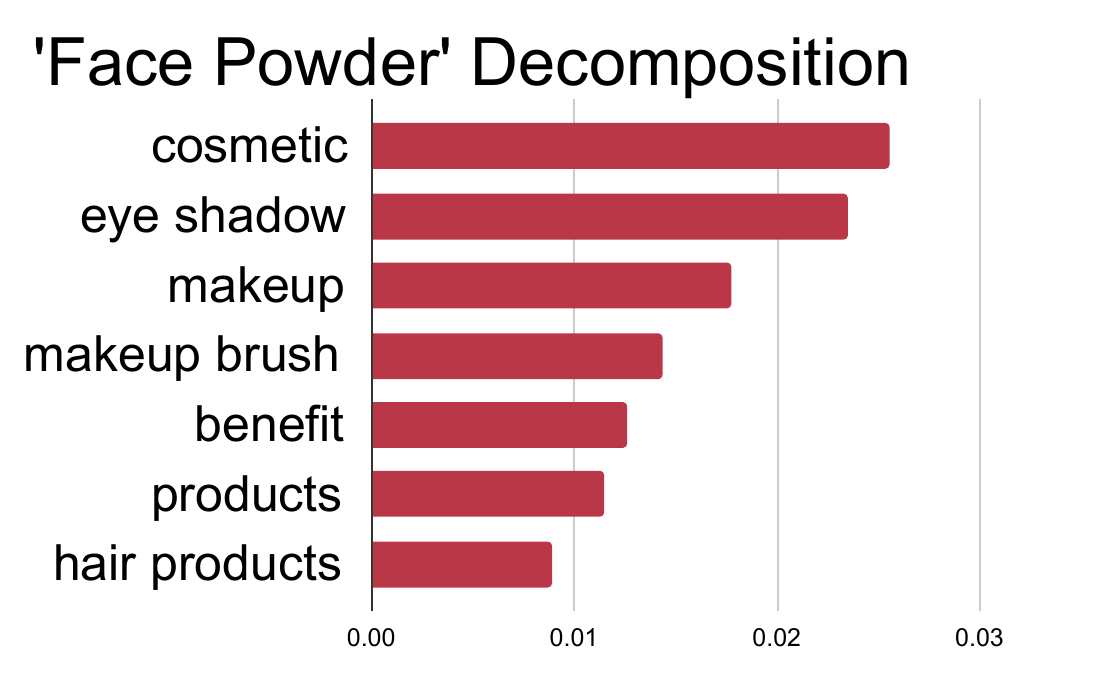}
    \end{subfigure}
    \begin{subfigure}{0.32\columnwidth}
        \includegraphics[width=\columnwidth]{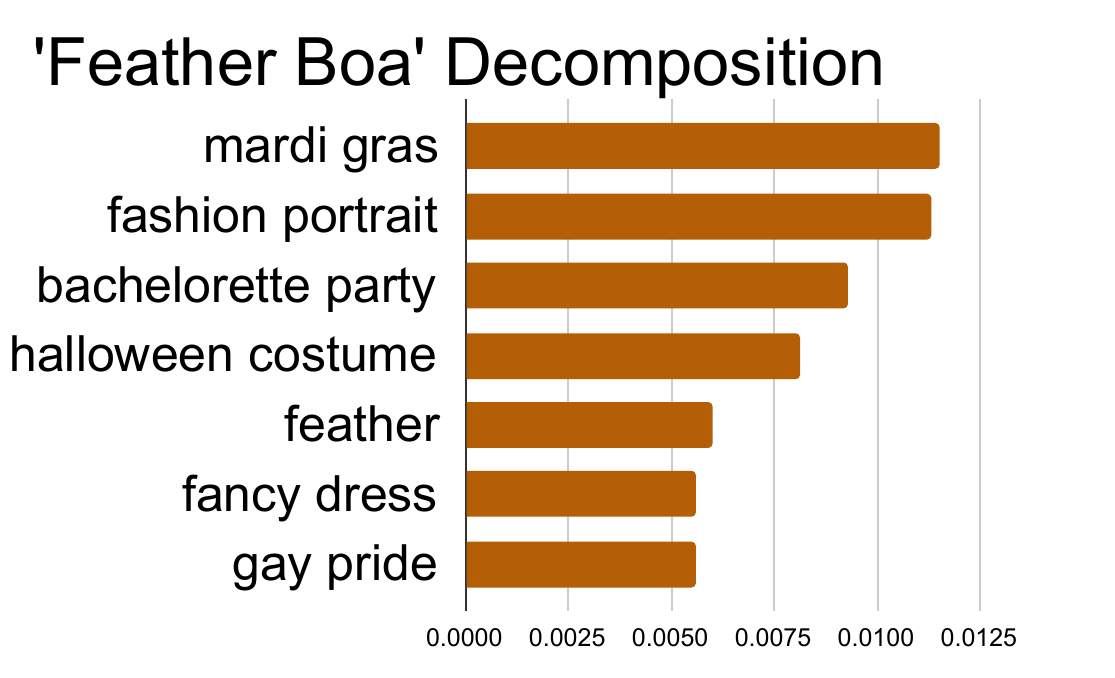}
    \end{subfigure}
    \begin{subfigure}{0.32\columnwidth}
        \includegraphics[width=\columnwidth]{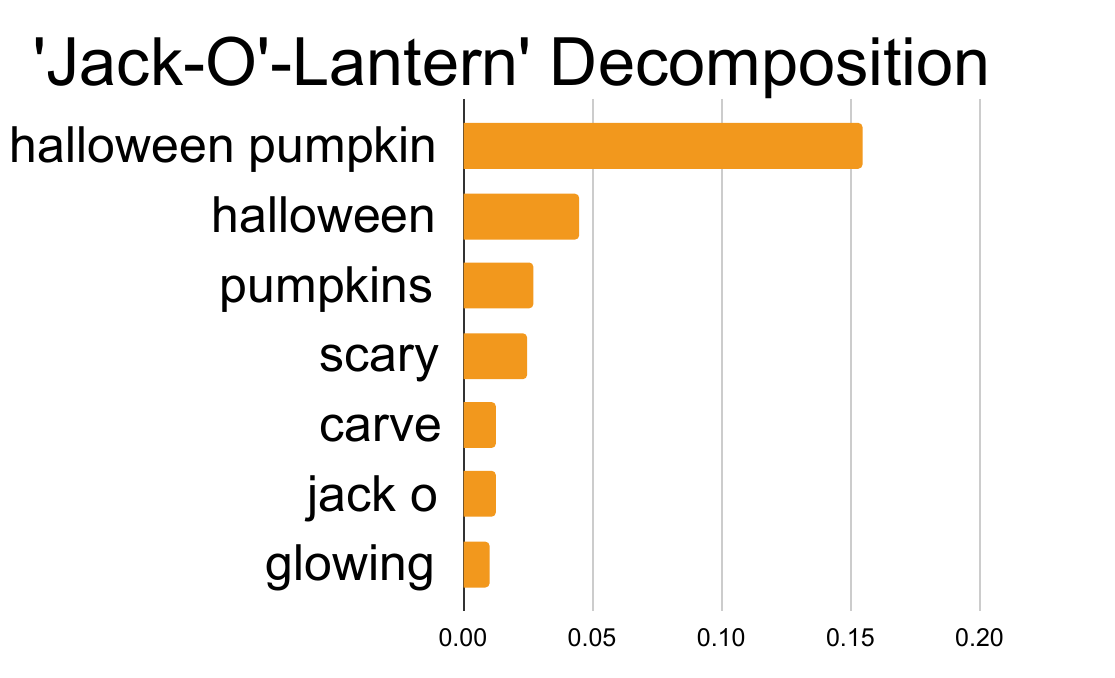}
    \end{subfigure}
    \begin{subfigure}{0.32\columnwidth}
        \includegraphics[width=\columnwidth]{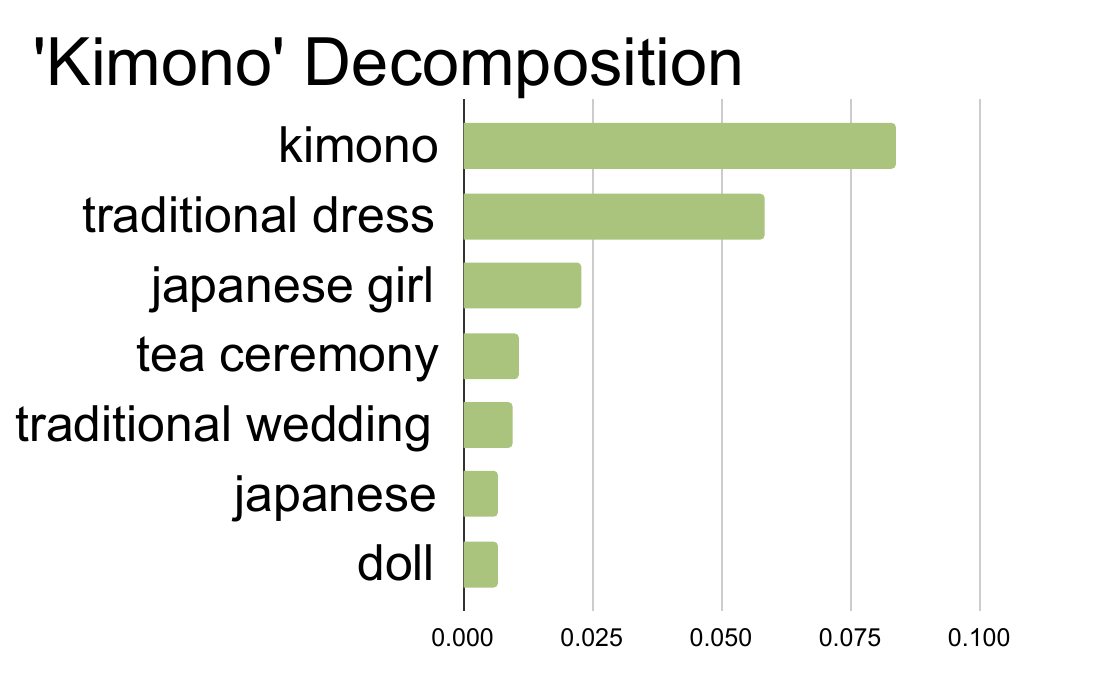}
    \end{subfigure}
    \begin{subfigure}{0.32\columnwidth}
        \includegraphics[width=\columnwidth]{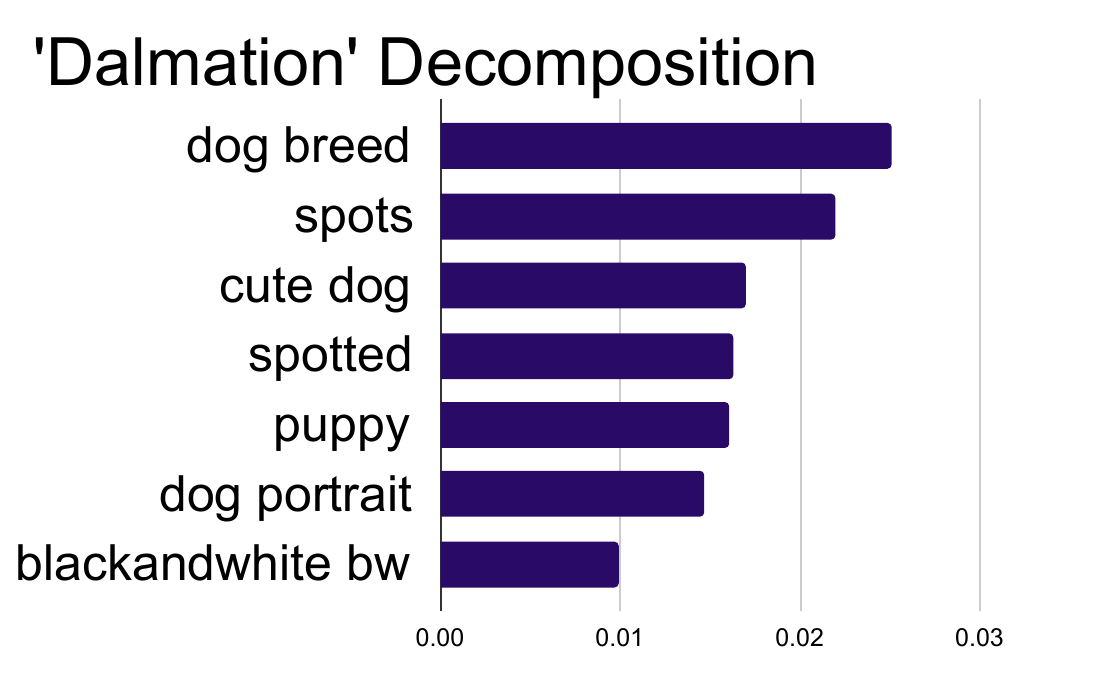}
    \end{subfigure}

    \caption{Example concept histograms of various ImageNet classes. The top seven concepts for each class are visualized along with their relative weighting, with the average $\ell_0$ norm of individual sample decompositions also being 7.}
    \label{fig:extra_clip_decomps}
\end{figure}

\subsection{Additional Case Study: Detecting Spurious Correlations}
\label{sec:app_spurious_correlations}
We present an additional case study for detecting spurious correlations in CIFAR100. In particular, we look at the prevalence of the spurious concept ``desert'' in the classes `camel' and `kangaroo' in Figure~\ref{fig:cifar_desert}. We observe that camels are more frequently pictured in the desert, creating a spurious signal that may be leveraged by downstream classifiers. This figure provides an additional example of how we can understand biases and trends in data with \method~decompositions. 

\begin{figure}
    \centering
    \includegraphics[width=0.58\textwidth]{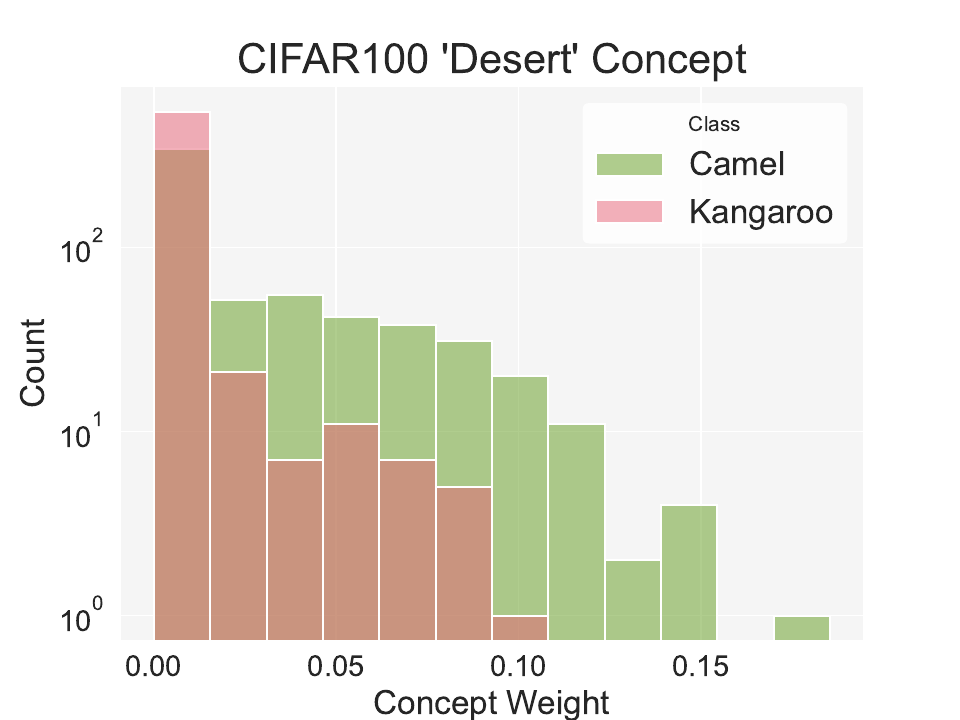}
    \caption{Distribution of ``Desert" concept in `Camel' and `Kangaroo' classes of CIFAR100.}
    \label{fig:cifar_desert}
\end{figure}

\subsection{Additional Case Study: Spurious Correlation Intervention}
\label{sec:app_wb_interv}
We further test the ability of \method to enable intervention on intermediate representations and linear classifiers by attempting to remove information pertaining to spurious signals. In particular, we consider the Waterbirds dataset \citep{sagawa2019distributionally}, which spuriously correlates landbirds with land backgrounds, resulting in trained classifiers performing poorly on waterbirds on land. We thus remove information about whether or not birds are on land backgrounds by ablating concept weights on ``bamboo'', ``forest'', ``hiking'', and ``rainforest'' as well as any bigrams containing the word ``forest,'' as shown in Table \ref{tab:interv_waterbirds}. This significantly improves worst-case subgroup performance for waterbirds on land from 0.48 to 0.60.

For both this experiment and the intervention on CelebA described in the main paper, we train linear probes using the LogisticRegressionClassifier module in scikit-learn using an $\ell_1$ penalty.

\label{sec:app_waterbirds}
    \begin{table}
    \centering
       \caption{\small Evaluation of intervention on spurious correlations for Waterbirds dataset. Removing information about land backgrounds improves worst-case subgroup performance.}
       \label{tab:interv_waterbirds}
       
       \begin{small}
       \begin{sc}
       \begin{tabular}{lll}
           \toprule
            &
             \begin{tabular}[c]{@{}l@{}}Landbirds\\ on land\end{tabular} &
             \begin{tabular}[c]{@{}l@{}}Waterbirds\\ on land\end{tabular} \\ \midrule
           Linear Probe &
             0.98 &
             0.48 \\
           \begin{tabular}[c]{@{}l@{}}Intervention \\ Probe\end{tabular} &
             0.97 &
             \textbf{0.60} \\ \bottomrule
       \end{tabular}
       \end{sc}
       \end{small}
    \end{table}

\subsection{Additional Case Study: Distribution Shift Monitoring}
\label{sec:app_dist_shift1}
We present a final case study using \method to monitor distribution shift. This can help identify differences between training and inference distributions or evaluate how a continually sampled dataset changes over time. In this experiment we consider the Stanford Cars dataset \citep{stark2011fine}, which contains photos of cars from 1991 to 2012, including their make and year labels. By decomposing photos of cars from each year, we can view how the distribution changed yearly. We visualize the weights of the concepts ``convertible" and ``yellow" from our decompositions, as well as the actual percentage of cars from each year that were convertibles or yellow in Figure \ref{fig:cars_dist}. Note the right-hand y-axis, corresponding to the weight of the given concept $c_i$ over the sum of the weights of all concepts $\sum_i c_i$, does not have a meaningful unit of measure or scale. We find that the trends in the groundtruth concept prevalence generally closely match that of the predicted/decomposed concepts, allowing us to visualize which years convertibles or yellow cars were popular or out-of-distribution with respect to other years. Most notably, we see that \method picks up on the out-of-distribution rise in popularity of brightly colored sports cars in the early 2000s.
\begin{figure}[h]
    \centering
        \includegraphics[width=0.58\textwidth]{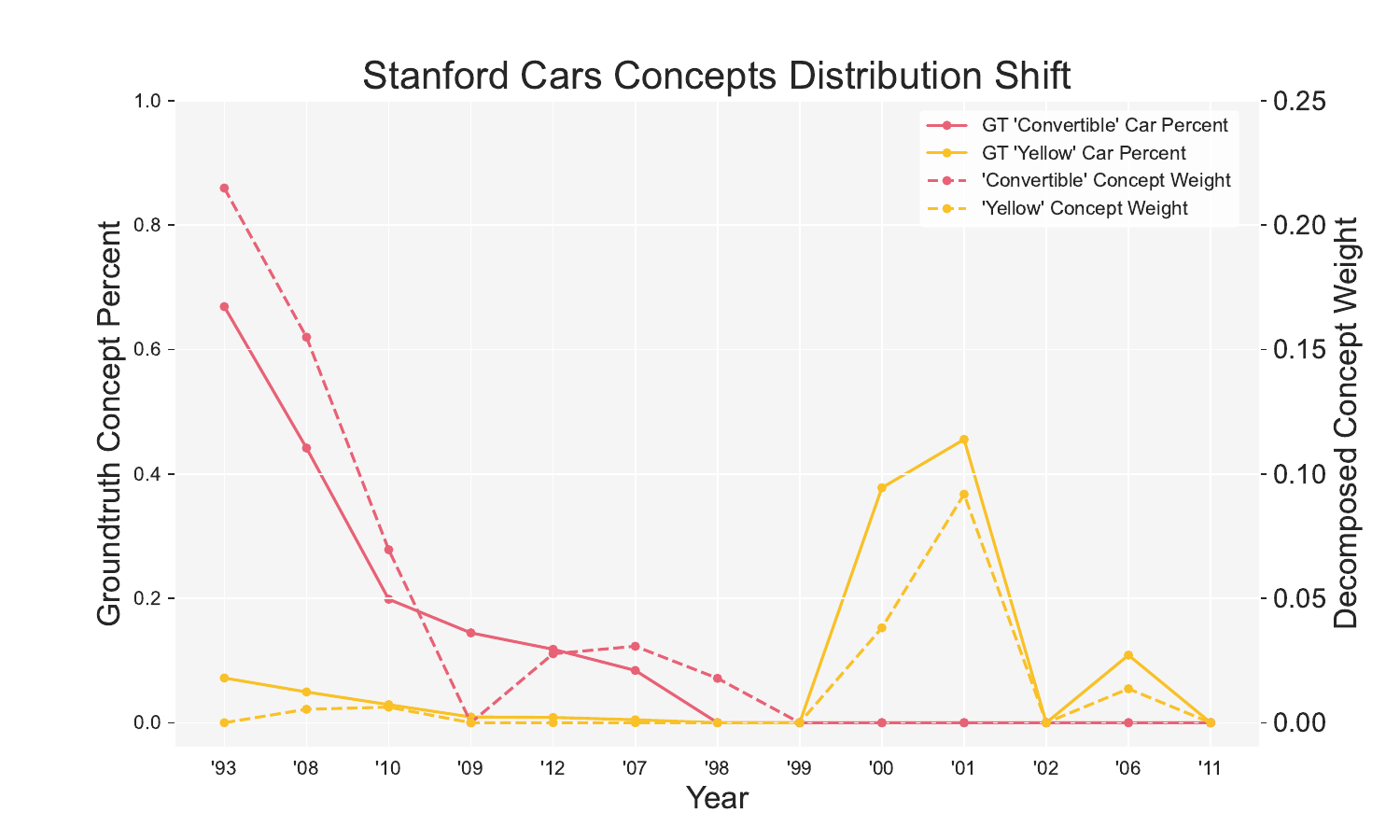}
        \caption{Visualization of the presence of convertibles (pink lines) and yellow cars (yellow lines) in Stanford Cars over time. \method concept weights (dotted) closely track the groundtruth concept prevalence (solid) for both concepts. }
        \label{fig:cars_dist}
\end{figure}

\subsection{Additional Case Study: Distribution Shift Monitoring}
\label{sec:app_dist_shift2}
To further verify that \method allows for identification and tracking of distribution shift, we study the Waterbirds dataset, which is known to have differently balanced train, valodation, and test splits. 
To identify distribution shifts, we can simply look at the norm of the difference between the class decompositions of the two classes for each splot, as shown in \ref{tab:wb1}. We find that the validation and test splits are much more similar than the training and validation splits or the training and test splits, which can be verified by the construction process of the Waterbirds dataset.

\begin{table}[]
\centering
\caption{Study of the differences in distributions between train, validation, and test splits of Waterbirds. The validation and test splits are much more similar to each other than they are to the train split.  }
\begin{small}
       \begin{sc}
\begin{tabular}{llll}
\toprule
                & Train, Val & Train, Test     & Val, Test      \\ \midrule
Class Landbird  & 0.0182     & 0.0182          & \textbf{0.005} \\
Class Waterbird & 0.0229     & .0188 & \textbf{0.009} \\ \bottomrule
\end{tabular}
\end{sc}
\end{small}
\label{tab:wb1}
\end{table}

We also find that the most weighted concept in the ‘landbird’ class of the train split is “bamboo” but the corresponding weight for “bamboo” in the ‘waterbird’ class is much lower. The “bamboo” concept weight for both classes and all splits is shown below, where we see that the validation and test splits are very similar and mostly evenly balanced, whereas the train split is highly unbalanced.

\begin{table}[]
\centering
\caption{Study of the prevalence of the concept ``bamboo'' in the different classes and splits of Waterbirds. }
\begin{small}
       \begin{sc}
\begin{tabular}{llll}
\toprule
                & Train           & Val   & Test  \\ \midrule
Class Landbird  & \textbf{0.0196} & 0.010 & 0.010 \\
Class Waterbird & \textbf{0.0007} & 0.008 & 0.008 \\ \bottomrule
\end{tabular}
\end{sc}
\end{small}
\end{table}

\subsection{Checking the Interpretability of Negative Concepts}
\label{sec:app_negweights}

We take a set of 71 concept-antonym pairs from the MIT States dataset and embed the terms in CLIP. With and without concept centering, we observe that these concept-antonym pairs have an average cosine similarity well above -1, indicating that CLIP does not place antonyms in opposite directions, as shown in \ref{tab:neg_antonyms}. Next, we take our concept dictionary and prepend ``not" to all of the words and compare the average cosine similarity between concept and not-concept pairs. Similarly, we observe that with and without centering, concept and not-concept pairs are highly similar. Note that the average similarity for true pairs of images and text in MSCOCO is less than the similarity between concepts and not-concepts with and without centering.

\begin{table}[h!]
\centering
\caption{Evaluation of the similarity of antonyms and negative concepts in CLIP.}
\begin{small}
       \begin{sc}
\begin{tabular}{lll}
\toprule
 &
  \begin{tabular}[c]{@{}l@{}}Pairwise Cosine Similarity\\ (without concept centering)\end{tabular} &
  \begin{tabular}[c]{@{}l@{}}Pairwise Cosine Similarity \\ (with concept centering)\end{tabular} \\ \midrule
Concept and antonym &
  0.7176 $\pm$ 0.1109 &
  0.1366 $\pm$ 0.2197 \\
Concept and “not” concept &
  0.8661 $\pm$ 0.0498 &
  0.6130 $\pm$ 0.0498 \\ \bottomrule
\end{tabular}
\end{sc}
\end{small}
\label{tab:neg_antonyms}
\end{table}

\subsection{Understanding the Image Mean for Modality Alignment}
\label{sec:app_mean}
In order to empirically check that the mean centering of images does not result in a loss of information, we decompose the img mean, $\mu_{img}$, that we used for all experiments. If we decompose it with uncentered concepts, the following concepts are highlighted: \{``closeup", ``flickr", ``posed"\}. The decomposition with centered concepts results in the following concepts: \{``flickr", ``posed", ``pics", ``angle view", ``last post"\}. These concepts all seem to be generally related to images, with minimal other semantic information, suggesting that centering does not remove any discriminative semantic content of embeddings, but simply removes information about the modality.

\subsection{Choice of Concept Vocabulary} 
\label{sec:app_vocab_choice}
We perform a simple ablation study to assess the sensitivity of our method to choices in concept vocabulary. We collect a second vocabulary in the same exact manner as the LAION vocabulary from the MSCOCO caption dataset. We consider both the top 10k and top 5k most common words for both, and repeat the zero-shot accuracy and reconstruction cosine similarity experiments from Section \ref{sec:perf} on CIFAR100. We see that the MSCOCO10k and LAION10k vocabularies perform almost exactly the same for both metrics. The smaller vocabularies perform the same for cosine reconstruction but underperform the 10k vocabularies for zero-shot classification tasks. 

    \begin{figure}[h]
        \centering
        \begin{subfigure}{0.45\columnwidth}
        \centering
            \includegraphics[width=\columnwidth]{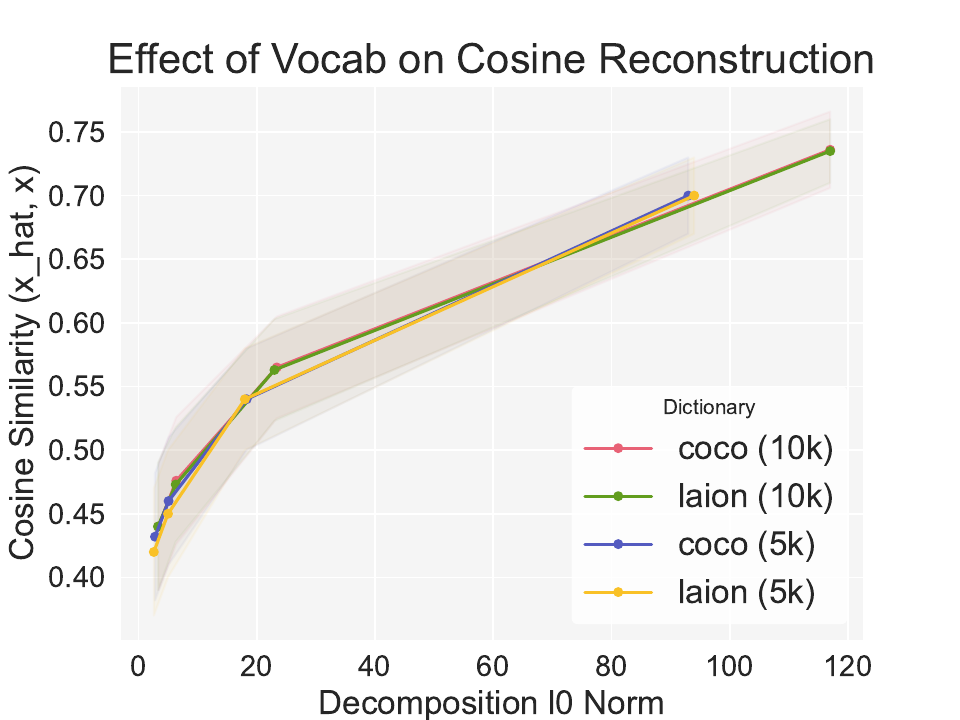}
        \end{subfigure}
        \begin{subfigure}{0.45\columnwidth}
        \centering
            \includegraphics[width=\columnwidth]{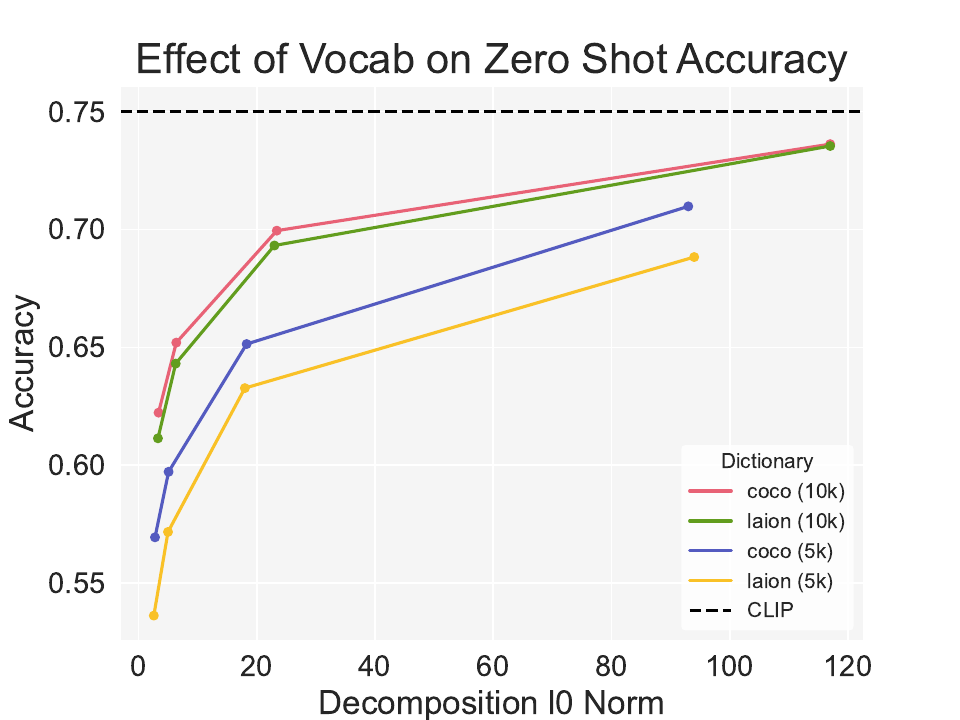}
        \end{subfigure}
        \caption{Change in \method performance when considering another semantic concept dictionary derived from MSCOCO as well as a smaller concept vocabulary.}
        \label{fig:vocab_abl}
    \end{figure}

\subsection{Concept Type Distribution}
\label{sec:app_concept_type_dist}
In order to better understand any biases produced by the decomposition process or that CLIP itself has, we visualize the types of concepts most commonly activated across multiple datasets, labelling them by part of speech in Figure \ref{fig:concept_pos}. We see that nouns are by far the most common concepts across datasets, indicating that both CLIP and the decompositions are highly object centric. Note that the low weight on verbs and adjective is due to far fewer concepts of those types being activated (low $l_0$ norm) as well as the weight upon those concepts being significantly smaller (low $l_1$ norm). We hypothesize that the information in many adjective and verbs can actually be encoded into the noun itself, resulting in this phenomenon. For example, the concept ``lemon" is a more succinct form of ``yellow" and ``fruit". 

    \begin{figure}[h]
    \centering
        \includegraphics[width=0.5\columnwidth]{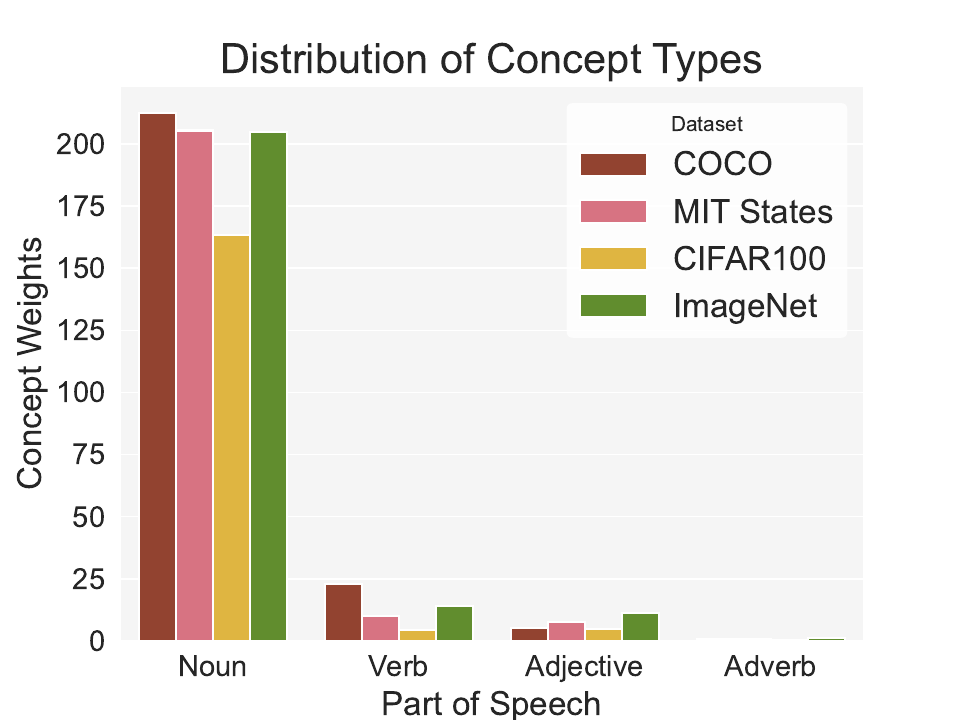}
        \caption{\method decompositions are mostly comprised of nouns across multiple datasets. } 
        \label{fig:concept_pos}
    
    \end{figure}

\subsection{Experiments on Alternative CLIP Architecture}
\label{sec:app_clip_rn50}
We present cosine reconstruction and zero-shot accuracy experiments with an alternative CLIP architecture from OpenAI with a ResNet50 backbone for the vision encoder. Note that these experiments were done with a 10000 size vocabulary of only one-word concepts. We find that results are similar to those presented in \ref{fig:cosine}, save for OpenAI's ResNet50 CLIP performing much worse than OpenCLIP's ViT B/32 backbone in general. 
\begin{figure}[!h]
    \centering
    \begin{subfigure}{0.32\columnwidth}
        \includegraphics[width=\columnwidth]{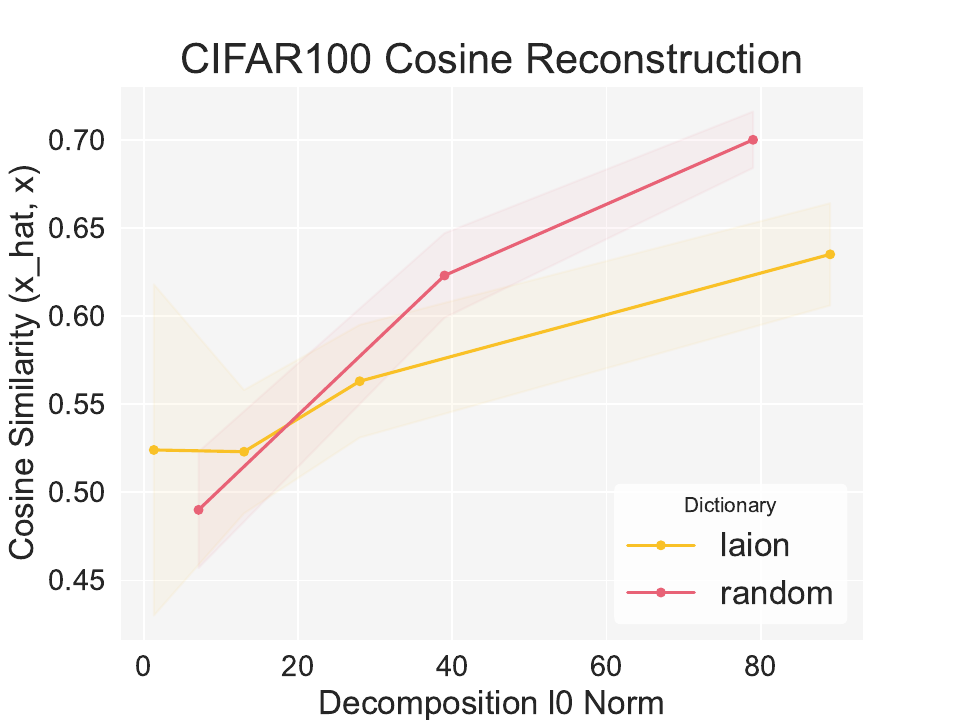}
    \end{subfigure}
    \begin{subfigure}{0.32\columnwidth}
        \includegraphics[width=\columnwidth]{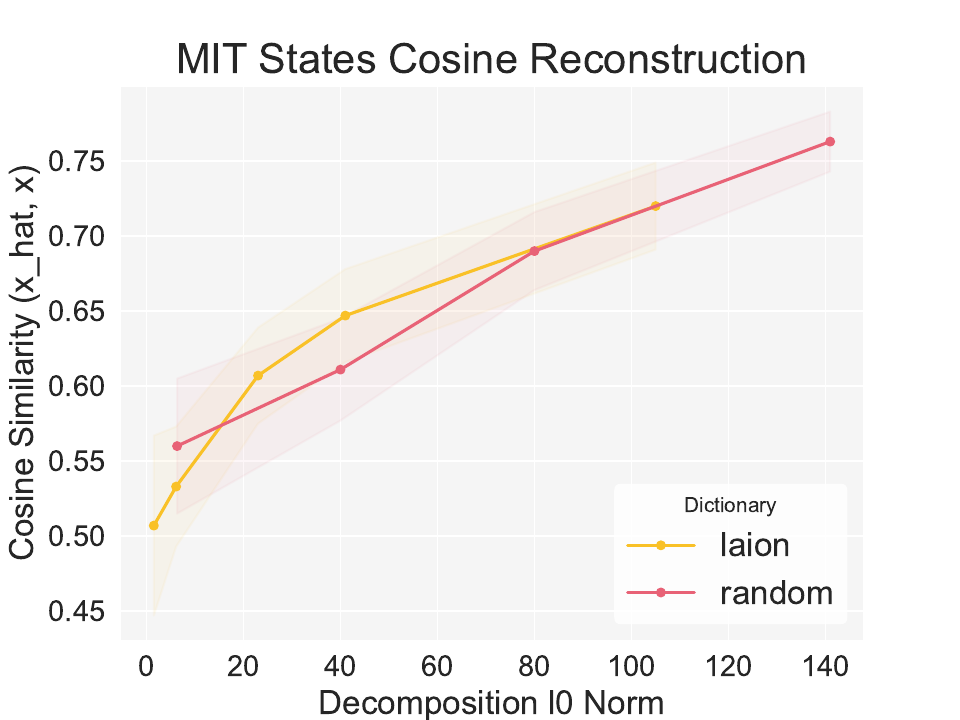}
    \end{subfigure}
    \begin{subfigure}{0.32\columnwidth}
        \includegraphics[width=\columnwidth]{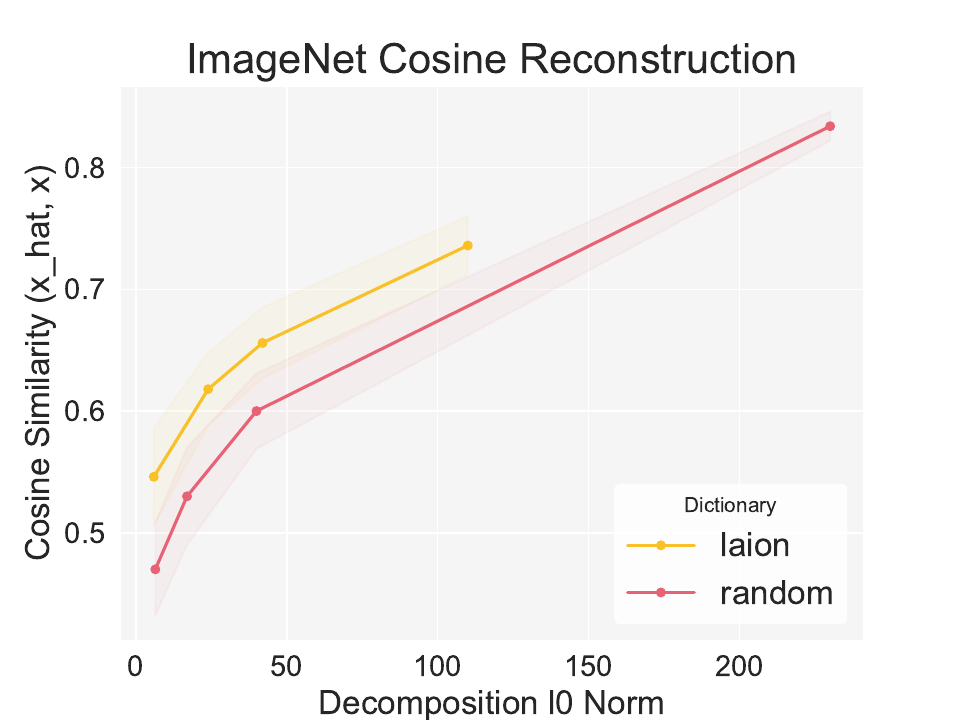}
    \end{subfigure}
    \begin{subfigure}{0.32\columnwidth}
        \includegraphics[width=\columnwidth]{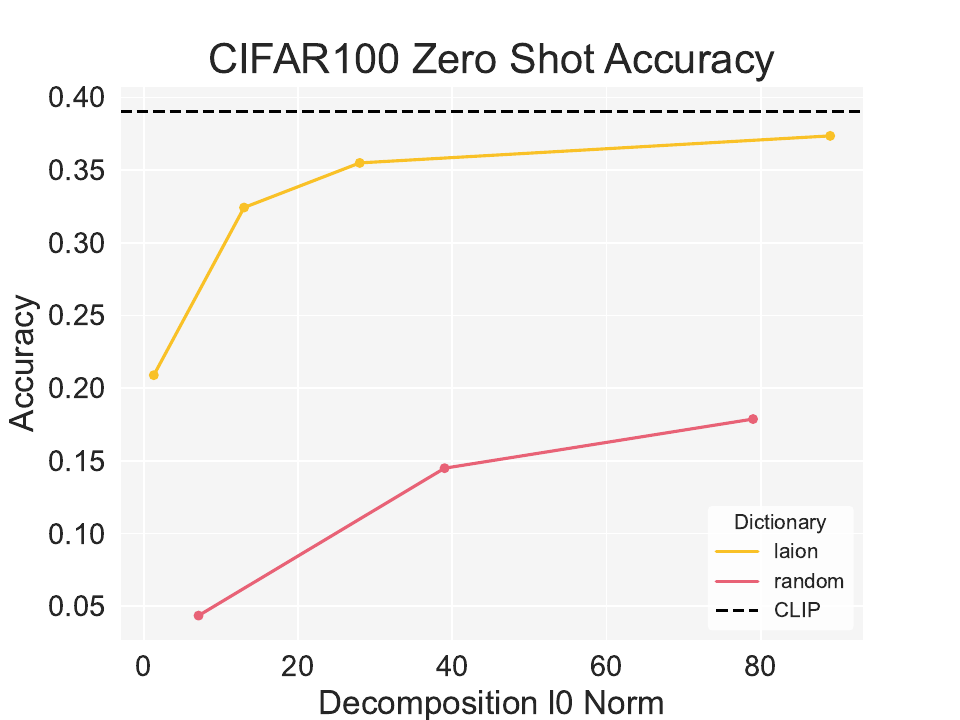}
    \end{subfigure}
    \begin{subfigure}{0.32\columnwidth}
        \includegraphics[width=\columnwidth]{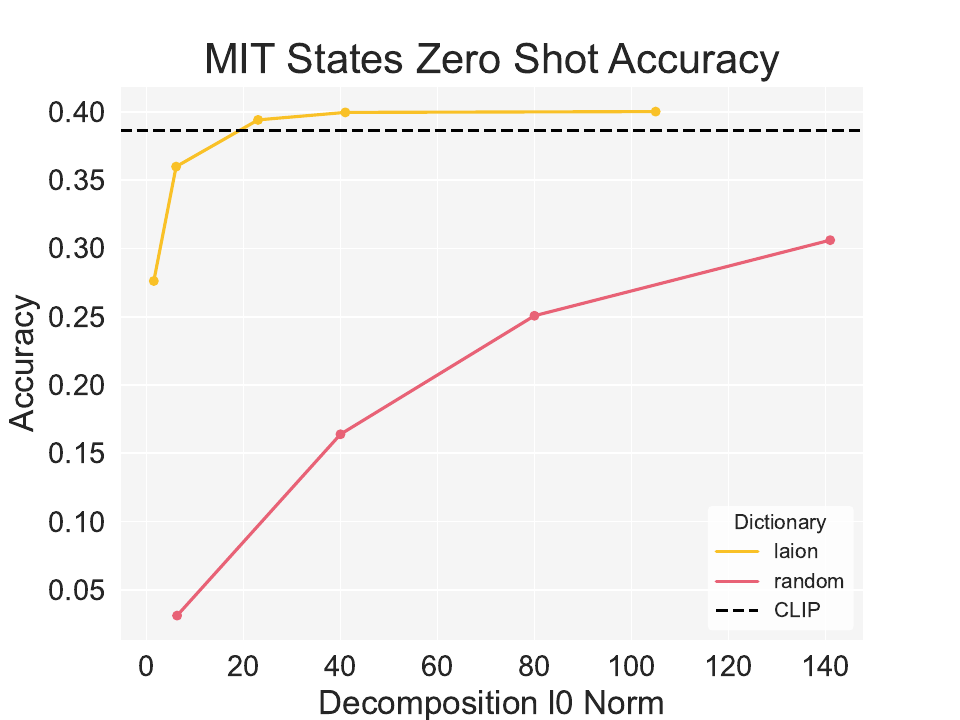}
    \end{subfigure}
    \begin{subfigure}{0.32\columnwidth}
        \includegraphics[width=\columnwidth]{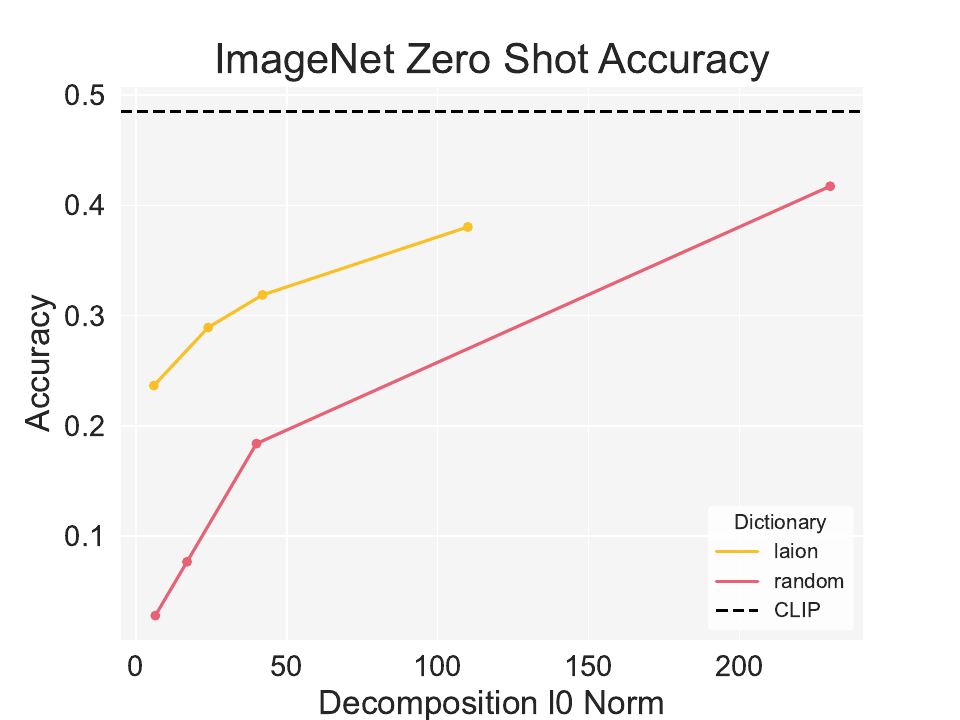}
    \end{subfigure}
    \caption{Performance of SpLiCE decomposition representations on zero-shot classification tasks (bottom row) and cosine similarity between CLIP embeddings and SpLiCE embeddings (top row) for OpenAI's ResNet50 CLIP model.}
    \label{fig:extra_clip}
\end{figure}


\end{document}